\documentclass[11pt]{article}

\usepackage{arxivpreprint}
\usepackage[margin=1in]{geometry}
\usepackage{graphicx}%
\usepackage{multirow}%
\usepackage{amsmath,amssymb,amsthm}
\usepackage{graphicx}
\usepackage[hidelinks]{hyperref}
\usepackage{xcolor}%
\usepackage{textcomp}%
\usepackage{manyfoot}%
\usepackage{booktabs}%
\usepackage{algorithm}%
\usepackage{algorithmicx}%
\usepackage{algpseudocode}%
\usepackage{listings}%
\usepackage{amsmath}
\usepackage{booktabs}
\usepackage{tabularx}
\usepackage{subcaption}
\usepackage{algorithm}
\usepackage{algpseudocode}
\usepackage{algorithmicx} 
\usepackage{xcolor}

\usepackage[margin=1in]{geometry}
\usepackage{microtype}              
\usepackage{lmodern}                
\usepackage[T1]{fontenc}
\usepackage{setspace}
\usepackage{titlesec}
\titlespacing*{\section}{0pt}{2.2ex plus .6ex minus .2ex}{1.2ex plus .3ex}
\titlespacing*{\subsection}{0pt}{1.8ex plus .5ex minus .2ex}{0.9ex plus .2ex}
\titlespacing*{\subsubsection}{0pt}{1.4ex plus .4ex minus .1ex}{0.7ex plus .2ex}

\usepackage{authblk}
\usepackage{fancyhdr}
\usepackage{lastpage}

\makeatletter
\renewcommand{\maketitle}{%
  \begin{center}
    \hrule height 0.8pt \vspace{0.8em}
    {\LARGE\bfseries \@title \par}
    \vspace{0.8em}
    {\large \@author \par}
    \vspace{0.6em}
    \hrule height 0.8pt
  \end{center}
  \vspace{1.2em}
}
\makeatother

\newtheorem{theorem}{Theorem}

\newtheorem{proposition}[theorem]{Proposition}%

\newtheorem{example}{Example}%
\newtheorem{remark}{Remark}%
\newtheorem{lemma}{Lemma}
\newtheorem{assumption}{Assumption}
\newtheorem{corollary}{Corollary}

\newtheorem{definition}{Definition}%

\title{Learning under Distributional Drift: Prequential
Reproducibility as an Intrinsic Statistical Resource}

\author[1]{Sofiya Zaichyk}
\affil[1]{Innovative Defense Technologies (IDT)}

\date{}

\begin{document}
\maketitle
\thispagestyle{plain}

\begin{abstract}
Statistical learning under distributional drift remains poorly characterized, especially in closed-loop settings where learning alters the data-generating law. We introduce an intrinsic drift budget $C_T$ that quantifies cumulative information-geometric motion of the data distribution along the realized learner--environment trajectory, measured in Fisher--Rao distance. The budget separates exogenous environmental change from policy-sensitive feedback induced by the learner's actions. This gives a rate-based characterization of prequential reproducibility: when performance on the realized stream is used to predict one-step-ahead performance under the next distribution, the drift contribution enters through the average motion rate $C_T/T$, not through cumulative drift alone. We prove a drift--feedback bound of order \(T^{-1/2}+C_T/T\), up to controlled second-order remainder terms, and establish a matching sharpness lower bound for the same prequential reproducibility gap on a canonical regular subclass. Thus the dependence on the average Fisher--Rao motion rate is tight up to constants: $C_T/T$ is sufficient for upper control and unavoidable on regular hard subclasses. We further prove an information-theoretic indistinguishability result showing that order-\(C/T\) effects on the one-step-ahead target need not be identifiable from the realized performance stream alone. Finally, we show that fixed monitoring channels induce contracted observable Fisher motion, and experiments, including a misspecified real-data feedback setting, indicate that appropriately chosen channels can retain risk-relevant drift signal when the intrinsic data-generating law is unavailable. The resulting theory treats exogenous drift, adaptive data analysis, and performative feedback as different sources of Fisher--Rao motion along the same learner--environment trajectory.
\end{abstract}
\vspace{0.5\baselineskip}
\vspace{0.5\baselineskip}

\section{Introduction}\label{sec:intro}

Modern learning systems often operate in \emph{self-modifying environments}, where learning alters the very distribution it learns from. A recommender policy reshapes user preferences as it updates \citep{Perdomo2020}; adaptive experiments change the distribution from which subsequent data are drawn \citep{li2024}; and reinforcement-driven agents modify the state transitions that determine future feedback \citep{Sutton1998,Dar2003}. In all such cases, the learner is not a passive observer of a stationary process but an active participant in an evolving distributional flow. The resulting feedback coupling breaks stationarity and, more importantly, the iid sampling premise underlying classical generalization theory~\citep{Vapnik1998}.

\subsection{Motivation}

Given samples $\{(x_i,y_i)\}_{i=1}^T$ drawn iid\ from a fixed distribution $p(x,y)$, population and empirical risks coincide asymptotically, with convergence rate $O(T^{-1/2})$~\citep{Vapnik1998}. Once the data-generating process itself changes, these guarantees collapse, and the population quantity being estimated becomes path-dependent. In such settings, the limiting factor is not only sample size, but the average \emph{one-step} change in the data-generating law.

As a running example, consider an adaptive recommender with policy $\pi_t$ that selects content while updating its model from observed responses. Each recommendation depends on past data yet also shapes future data, since user preferences evolve with exposure. Let the environment be parameterized by $\theta_t$, governing $p_{\theta_t}(x,y)$, and evolve in closed loop according to $\theta_{t+1}=F(\theta_t,u_t,\eta_t)$, where $u_t$ is some action chosen from a learned policy and $\eta_t$ denotes exogenous influences. Even when $F$ is smooth and deterministic conditional on $(u_t,\eta_t)$, its dependence on $u_t$ couples learning to the data-generating law. This relationship is visualized in Figure~\ref{fig:feedbackloop}. The induced distributions $p_{\theta_1},p_{\theta_2},\ldots$ trace a distributional trajectory. To reason about how rapidly this trajectory changes, we need a notion of \emph{local} statistical displacement between nearby $p_\theta$'s. The family $\{p_\theta\}$ forms a statistical manifold whose intrinsic Riemannian metric---the Fisher information---measures how sensitively the distribution changes under infinitesimal perturbations of $\theta$. Thus the learner--environment interaction induces geometric motion in Fisher--Rao distance, and the realized sequence $\{(x_t,y_t)\}$ need not be stepwise self-consistent. Even when the learner's update rule is fixed, the coupled dynamics can move the underlying law $p_{\theta_t}$ enough that performance measured at time $t$ is not predictive of performance under $p_{\theta_{t+1}}$. When Fisher--Rao motion accumulates rapidly in directions that affect the
deployed predictors' risks, successive one-step population targets can become
poorly aligned with realized performance, degrading prequential evaluation---i.e., sequential predictive evaluation along the realized trajectory, in the sense of the prequential framework of \citet{Dawid1999}.

\begin{figure}[t]
\centering
\includegraphics[width=1.0\linewidth]{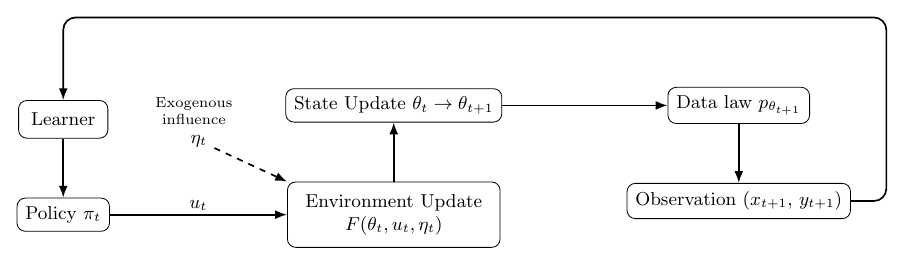}
\caption{
The learner’s policy $\pi_t$ and exogenous influence $\eta_t$ act on the environment, evolving $\theta_t$ to $\theta_{t+1}$ under $F(\theta_t,u_t,\eta_t)$. This new state defines the next data distribution $p_{\theta_{t+1}}$, closing the feedback loop. Exogenous factors $\eta_t$ perturb $\theta_t$ externally, while endogenous feedback arises from the learner’s own actions.}
\label{fig:feedbackloop}
\end{figure}

The same feedback-driven coupling arises in adaptive testing, model-in-the-loop simulation, and autonomous control. This raises a basic question: how rapidly can the learner--environment system move before classical guarantees break? To answer this, we develop a geometric analysis for quantifying distributional motion during learning. We then relate this motion to nonasymptotic generalization and prequential reproducibility under drift.

\subsection{Conceptual View}

We view the data-generating process as a trajectory $\{\theta_t\}$ on a statistical manifold $(\Theta,g_\theta)$ endowed with the Fisher--Rao metric \citep{AmariNagaoka2000}. Any smooth divergence admits a second-order expansion whose Hessian defines a local Riemannian metric \citep{Amari2016}. For the broad class of $f$-divergences, including KL, $\chi^2$, Hellinger, and $\alpha$-divergences, this induced metric coincides with the Fisher--Rao geometry. Fisher--Rao is therefore the natural information-geometric metric on a statistical model family \citep{cencov1982,AmariNagaoka2000}, capturing displacement intrinsic to the family rather than to a particular coordinate system. In particular, Fisher--Rao is the unique Riemannian metric on parametric families that is invariant under smooth reparameterizations \citep{cencov1982}. Because drift unfolds along a realized trajectory, we need a cumulative measure of motion that composes across time. Once a local geometry $g_\theta$ is fixed, the natural choice is arc length, $\int \|\dot \theta(t)\|_{g_{\theta(t)}}\,dt,$
which is additive under concatenation of path segments. Taking $g_\theta$ to be Fisher--Rao thus yields an intrinsic scale for describing how small distributional perturbations alter the statistical environment of a learner.

This finite-dimensional notation should be read as the regular theorem setting, not as a claim that real data-generating processes are intrinsically low-dimensional. More generally, a closed-loop learning process induces a trajectory of conditional laws
\[
P_t=\mathcal L(Z_t\mid \mathcal F_{t-1}), \qquad Z_t=(X_t,Y_t),
\]
which may be implicit, high-dimensional, or nonparametric. Fisher--Rao geometry can be defined directly on regular dominated families of densities, for tangent perturbations with finite Fisher norm. We work here with a finite-dimensional statistical manifold to keep the drift--feedback decomposition and local comparison arguments explicit. The observable-channel construction in Section~\ref{sec:observable_footprints} later provides the corresponding Fisher geometry for pushed-forward laws available to a monitor.

In discrete time, the intrinsic Fisher--Rao path length of the realized trajectory is
\begin{equation}\label{eq:intro_AT}
\mathcal{A}_T \;:=\; \sum_{t=1}^{T} d_F(\theta_{t+1},\theta_t),
\end{equation}
where $d_F$ is the Fisher--Rao geodesic distance. This is conceptually distinct from aggregating pairwise TV, Wasserstein, or KL discrepancies. Those are pairwise divergences not an intrinsic line element whose path integral yields a coordinate-free notion of cumulative motion. Thus $\mathcal A_T$ measures the statistical length of the realized path, rather than a collection of pairwise discrepancies in an externally chosen divergence.

\paragraph{Cumulative vs.\ per-step motion.}
The path length \(\mathcal A_T\) is cumulative and may grow even when the trajectory stays in a bounded region. Prequential reproducibility is assessed step-to-step, so the relevant distinction is between total path length and the average rate at which adjacent laws change.

\begin{definition}[Prequential reproducibility and intrinsic drift budget]
\label{def:reproducibility}
Prequential reproducibility is an \emph{evaluation} property of learning under an evolving data-generating law: the extent to which performance measured on the realized data stream remains predictive of performance under the \emph{next} data-generating distribution along the trajectory. In Section~\ref{sec:setup} we formalize this as a gap between empirical evaluation and a one-step-ahead (prequential) population target.

We summarize cumulative distributional motion by a tractable \emph{intrinsic drift budget}
\begin{equation}\label{eq:intro_CT}
C_T \;:=\; \sum_{t=1}^{T}\bigl(d_t+\alpha\,\kappa_t^{(\mathcal M)}\bigr),
\end{equation}
where $d_t$ captures exogenous drift and $\kappa_t^{(\mathcal M)}$ captures the leading-order policy-sensitive contribution in Fisher geometry. In Sections~\ref{sec:setup}--\ref{sec:drift_feedback} we prove that $C_T$ controls the intrinsic Fisher--Rao path length $\mathcal A_T$ up to a controlled second-order remainder. Accordingly, the relevant rate associated with this budget is $C_T/T$, the average Fisher--Rao motion per step.
\end{definition}

The distinction between $\mathcal A_T$ and $C_T$ is that $\mathcal A_T$ records the realized intrinsic motion, while $C_T$ decomposes the sources of that motion. Exogenous change contributes through $d_t$, and policy-sensitive feedback contributes through $\kappa_t^{(\mathcal M)}$, with both measured in the same local Fisher geometry. Figure~\ref{fig:two-panel-drift} illustrates this local decomposition and its accumulation into a global path length.

\begin{figure}[t]
\centering

\begin{subfigure}[t]{0.49\textwidth}
    \centering
    \includegraphics[width=\linewidth]{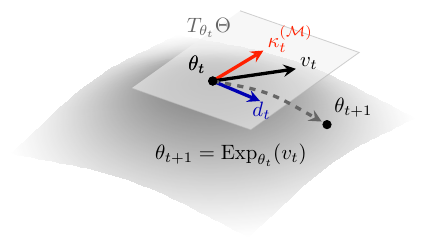}
    \caption{Local drift decomposition in a tangent space.}
    \label{fig:local}
\end{subfigure}
\hfill
\begin{subfigure}[t]{0.49\textwidth}
    \centering
    \includegraphics[width=\linewidth]{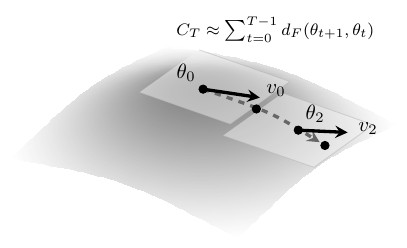}
    \caption{Global Fisher--Rao path length $\mathcal{A}_T$ and proxy budget $C_T$.}
    \label{fig:global}
\end{subfigure}

\caption{
Geometric intuition for learning under drift.
(a) At each step the environment parameter $\theta_t$ lies on the statistical manifold $(\Theta,g_\theta)$. Local motion decomposes into an \emph{exogenous} component, captured by $d_t$, and a \emph{policy-sensitive} component, captured by $\kappa_t^{(\mathcal M)}$, defined using local Fisher geometry around the baseline next state.
(b) Iterating these local movements generates a Fisher--Rao trajectory
$\theta_1 \to \theta_2 \to \cdots \to \theta_{T+1}$ with intrinsic length
$\mathcal{A}_T=\sum_{t=1}^{T}d_F(\theta_{t+1},\theta_t)$. The decomposed budget
$C_T=\sum_{t=1}^{T}(d_t+\alpha\kappa_t^{(\mathcal M)})$ upper-bounds this length, and the normalized rate $C_T/T$ is the quantity entering the prequential reproducibility bounds.
}
\label{fig:two-panel-drift}
\end{figure}

\paragraph{Coupling and effective drift.}
Because the target is one-step-ahead, the relevant quantity is the average rate of adjacent-law motion, $C_T/T$, rather than the cumulative budget alone. In feedback-driven settings this rate is not merely a property of an exogenous drift path. Policies that stably summarize the environment can reduce effective one-step change, whereas policies that chase transient fluctuations can sustain large per-step motion. The point is not that total motion must remain small, but that prequential reproducibility becomes drift-dominated when this one-step motion rate induces nonnegligible predictor-fixed risk variation over the evaluation horizon.

Our main results relate the drift-induced component of prequential behavior to the average motion rate $C_T/T$, and we prove a matching sharpness lower bound for the same prequential reproducibility gap on a canonical regular subclass, establishing a prequential reproducibility speed limit of order \(\Theta(T^{-1/2}+C/T)\).

\subsection{Relation to Prior Frameworks}

Several established frameworks capture pieces of the broader picture studied here. While each captures an important aspect of instability, the difficulty is that these pieces are usually treated in isolation. Our work provides a geometric viewpoint that connects them.

\emph{Stationary learning}~\citep{Vapnik1998,shalev2014} assumes fixed $p(x,y)$ and produces the familiar $O(T^{-1/2})$ convergence guarantees derived from concentration under iid\ sampling. 
\emph{Nonstationary stochastic optimization}~\citep{Besbes2015} introduces a variation budget that measures how much the target distribution can change exogenously over time, typically through a cumulative constraint such as $\sum_{t=1}^{T-1} \mathsf{d}(p_{\theta_{t+1}},p_{\theta_t}) \le V_T$ for a discrepancy $\mathsf{d}$ (often total variation or Wasserstein). Such budgets constrain external motion but they do not model endogenous coupling to the learner’s policy, nor do they encode how drift is distributed over time. In contrast, our bounds control prequential reproducibility through an \emph{average drift rate} $C_T/T$ derived from intrinsic Fisher--Rao motion along the realized trajectory, and in feedback-driven settings this rate can itself depend on the learner’s actions. Related approaches bound performance under drifting distributions using stability or Rademacher techniques~\citep{Blanchard2021,kuznetsov16,Mohri2012} but likewise do not explicitly model feedback between learner and environment. PAC--Bayes analyses have also been developed for distributional change, notably in the domain adaptation setting of \citet{Germain2016}. However, those results bound a single \emph{source--target shift}, whereas the present work studies \emph{path-dependent}, cumulative drift induced by sequential learner--environment coupling.

\emph{Performative prediction}~\citep{Perdomo2020} explicitly introduces feedback between model and environment, focusing on equilibrium conditions where that feedback vanishes ($p_{\theta_{t+1}}\approx p_{\theta_t}$). Follow-up work has refined aspects of this coupling~\citep{izzo21,jagadeesan22}, but continues to analyze convergence to fixed points rather than the explicit accumulation of temporal drift.
\emph{Adaptive data analysis}~\citep{Dwork2015,rogers20,russo2016,bassily2016} addresses feedback in the statistical query setting and bounds overfitting through mutual-information or stability arguments, yet does so within a static population that does not itself evolve.

These approaches define distinct ways to measure instability: sample variance, variation budgets, equilibrium deviation, and information flow. The present work connects them by treating drift as motion on a statistical manifold. The total motion decomposes into an exogenous component and a policy-sensitive component,
\begin{equation}\label{eq:intro_drift_terms}
\underbrace{d_t}_{\text{exogenous drift}}
\quad+\quad
\underbrace{\kappa_t^{(\mathcal M)}}_{\text{policy-sensitive drift}},
\end{equation}
and the resulting budget \eqref{eq:intro_CT} recovers classical regimes as limiting cases distinguished by which source of motion dominates. In the purely exogenous case $\kappa_t^{(\mathcal M)}\equiv 0$, the normalized budget $C_T/T$ plays the same operational role as a variation budget \emph{rate}, but is tied to Fisher--Rao motion and to the induced one-step risk shifts rather than imposed in an external discrepancy. A summary of these relationships is provided in Table~\ref{tab:related-regimes}.

\begin{table}[h]
\centering
\small
\caption{Learning regimes recovered as limits of the drift--feedback bound. Each corresponds to nullifying either the exogenous drift $d_t$ or the policy-sensitivity term $\kappa_t^{(\mathcal M)}$.}
\label{tab:related-regimes}
\vspace{6pt}
\begin{tabular}{p{0.21\textwidth} p{0.20\textwidth} p{0.26\textwidth} p{0.25\textwidth}}
\toprule
\textbf{Learning regime} & \textbf{Active terms} &
\textbf{Representative work} & \textbf{Notes} \\
\midrule
Stationary \emph{iid} &
$\sqrt{T}$ term only &
\cite{Vapnik1998} &
No drift or feedback; fixed $p(x,y)$. \\[3pt]

Exogenous drift &
$d_t$ term &
\cite{Besbes2015} &
Distribution changes externally over time. \\[3pt]

Performative equilibrium &
$d_t \!\to\! 0,\ \kappa_t^{(\mathcal M)} \!\to\! 0$ &
\cite{Perdomo2020} &
Coupled learner--environment system converging to a performative equilibrium. \\[3pt]

Adaptive data analysis &
$\kappa_t^{(\mathcal M)}$ term &
\cite{Dwork2015} &
Feedback/adaptivity-driven instability. \\
\bottomrule
\end{tabular}
\end{table}

A broader online-optimization literature studies nonstationarity via \emph{tracking} objectives and bounds performance through dynamic regret, which scales with the path length of a time-varying comparator sequence and reduces to static regret in the stationary limit \citep{Hall2013,Zhao2024}. These results quantify how well an algorithm can follow a moving target, but they are typically stated in terms of loss-sequence variation or comparator drift and do not address prequential reproducibility gaps induced by \emph{endogenous} distributional evolution. Our analysis is complementary. It identifies an intrinsic, distribution-level motion rate $C_T/T$ that controls step-to-step population risk shifts along the realized closed-loop trajectory.

Information-theoretic analyses of adaptive learning express related limits in terms of mutual-information accumulation between data and model parameters~\citep{russo2016,xu2017}. In statistical physics, learning has been interpreted as a thermodynamic process in which information gain is constrained by entropy production \citep{Still2012,Goldt2017,ortega2013thermodynamics}. All of these perspectives describe limited information flow, but they do so with external metrics or equilibrium arguments rather than an intrinsic geometric surrogate tied to Fisher motion.

Finally, the geometric ingredients we use---Fisher information as the canonical invariant metric and the local quadratic expansion of KL---are classical in information geometry \citep{cencov1982,AmariNagaoka2000,Amari2016}. A related line of work uses this geometry primarily to motivate \emph{algorithms} (e.g., natural-gradient or Riemannian preconditioning) and adaptive learning dynamics, including in nonstationary or changing environments \citep{Amari2019,Murata2002}. Our focus is different. Rather than parameter space, we use it to quantify \emph{intrinsic motion of the data-generating law} and to derive nonasymptotic prequential reproducibility limits governed by the drift \emph{rate} $C_T/T$, including a matching lower bound showing sharpness of the drift-rate
dependence for the prequential reproducibility gap itself. Section~\ref{sec:observable_footprints} further adds an observability principle which isolates what can and cannot be inferred about intrinsic drift from partial observations.

\subsection{Contributions and Significance}

The drift budget \(C_T\) decomposes Fisher--Rao motion in a closed-loop learning process into exogenous and policy-sensitive contributions. Its normalized rate \(C_T/T\), rather than cumulative motion alone, is the quantity that enters the prequential bound. This lets exogenous drift and policy-sensitive feedback be treated in the same units while keeping their contributions separate.

The main contributions of this paper are:
\begin{enumerate}
    \item We develop a closed-loop model for learning under endogenous drift, modeling the data-generating process as a trajectory on a statistical manifold and formalizing prequential reproducibility as a one-step-ahead gap along the realized learner--environment path.

    \item We define drift primitives $(d_t,\kappa_t^{(\mathcal M)})$ and a decomposed Fisher--Rao drift budget $C_T$ that separates exogenous motion from policy-sensitive motion within Fisher geometry, with operational control governed by the average rate $C_T/T$.

    \item We establish finite-sample bounds that separate on-trajectory sampling concentration from drift-induced risk variation, yielding explicit exogenous and policy-sensitive contributions and the scaling
    \[
    \mathbb E\,\Delta_T^{\mathrm{rep}} \lesssim T^{-1/2}+C_T/T.
    \]

    \item We prove two separate lower bounds on canonical regular drift--feedback subclasses. A same-object sharpness lower bound establishes that $\Theta(T^{-1/2}+C/T)$ is the sharp rate for $\mathbb E\Delta_T^{\mathrm{rep}}$ itself, while a Fano construction shows that order-$C/T$ effects on the one-step-ahead target can remain statistically hard to distinguish from the realized prequential evidence.

    \item We relate the intrinsic theory to observable monitoring channels. Fixed channels induce contracted Fisher motion on their output laws, giving a channel-dependent view of drift. Experiments in Gaussian, nonlinear teacher--learner, and misspecified real-data feedback settings show when this projected motion remains aligned with one-step risk shifts.
\end{enumerate}

The remainder of the paper is organized as follows.
Section~\ref{sec:preliminaries} introduces preliminary tools used in the analysis;
Section~\ref{sec:setup} formalizes the endogenously drifting process and its metric structure;
Section~\ref{sec:drift_feedback} defines the drift measures and intrinsic drift budget;
Section~\ref{sec:theory} presents the main theorems and lower bounds;
Section~\ref{sec:observable_footprints} develops observable Fisher--Rao drift rates under monitoring channels;
Section~\ref{sec:experiments} reports empirical validation; and Section~\ref{sec:discussion} discusses implications and extensions.

\section{Preliminaries and Lemmas}
\label{sec:preliminaries}

This section fixes notation and local assumptions used later. The main prequential gap and drift decomposition are introduced in Section~\ref{sec:setup}; the present section only supplies the geometric and probabilistic background needed for the bounds. We define the Fisher--Rao metric, state the filtration convention under which the loss residuals are martingale differences, impose local regularity along the realized trajectory, and record the concentration and path-length facts used in Sections~\ref{sec:drift_feedback}--\ref{sec:theory}. Full proofs of nonstandard or technically involved statements are deferred to Appendix~\ref{app:prelim_proofs}.

\subsection{Measure--metric foundations}

Let \((\mathcal X,\mathcal Y)\) denote the input--output space with the product \(\sigma\)--algebra. For the main theorem we work in a regular local statistical model \(\{p_\theta(x,y):\theta\in\Theta\}\), where \(\Theta\) is a smooth finite-dimensional parameter manifold. The analysis is local along the realized learner--environment trajectory; the assumptions below require only that the visited states lie in a compact regularity region \(K\subset\Theta\). Expectations under $p_\theta$ are written $\mathbb E_\theta[\cdot]$.

Throughout, expressions of the form $\theta+\Delta\theta$ are understood in a fixed local coordinate chart. When we require an intrinsic tangent displacement between two nearby points $\theta,\theta'\in\Theta$, we write
\[
\log_\theta(\theta') \in T_\theta\Theta
\]
for the Riemannian logarithm map (defined on the normal neighborhood of $\theta$), so that $\operatorname{Exp}_\theta(\log_\theta(\theta'))=\theta'$.

\begin{definition}[Fisher--Rao metric and distance]
\label{def:fisher-rao}
For each $\theta\in\Theta$, let $T_\theta\Theta$ denote the tangent space at $\theta$. The \emph{Fisher--Rao metric} is the Riemannian metric $g_\theta:T_\theta\Theta\times T_\theta\Theta\to\mathbb R$ defined by
\[
  g_\theta(v,w)
  ~=~
  \mathbb{E}_\theta\!\left[
    \langle\nabla_\theta\log p_\theta(x,y),v\rangle\,
    \langle\nabla_\theta\log p_\theta(x,y),w\rangle
  \right],
  \qquad v,w\in T_\theta\Theta.
\]
In local coordinates $\theta=(\theta^1,\dots,\theta^d)$, with $\partial_i=\partial/\partial\theta^i$, this gives the Fisher information matrix
\[
  g_{ij}(\theta)
  ~=~
  \mathbb{E}_\theta\!\left[
    \partial_i\log p_\theta(x,y)\,\partial_j\log p_\theta(x,y)
  \right].
\]
For $v\in T_\theta\Theta$, the induced norm is
\[
  \|v\|_{g_\theta}
  ~=~
  \sqrt{g_\theta(v,v)}.
\]

The associated \emph{Fisher--Rao geodesic distance} between $\theta,\theta'\in\Theta$ is
\[
  d_F(\theta,\theta')
  ~=~
  \inf_{\gamma\in\Gamma(\theta,\theta')}
  \int_0^1 \|\dot\gamma(s)\|_{g_{\gamma(s)}}\,ds,
\]
where $\Gamma(\theta,\theta')$ is the set of piecewise--$C^1$ curves $\gamma:[0,1]\to\Theta$ with $\gamma(0)=\theta$ and $\gamma(1)=\theta'$. See \citet{rao1945} and \citet{AmariNagaoka2000}.
\end{definition}

\subsection{Filtration and sampling semantics}

We work with the natural filtration generated by past observations and actions,
\[
\mathcal F_t=\sigma(x_1,y_1,u_1,\ldots,x_t,y_t,u_t).
\]
Throughout, the learner may update its predictor and policy using the available history. Concretely, we assume:
\begin{enumerate}
\item the environment parameter $\theta_t$ is $\mathcal F_{t-1}$-measurable;
\item the predictor $f_t$ and control $u_t$ are $\mathcal F_{t-1}$-measurable (or are drawn from conditional distributions given $\mathcal F_{t-1}$); and
\item conditional on $\mathcal F_{t-1}$, the next observation satisfies $(x_t,y_t)\mid \mathcal F_{t-1}\sim p_{\theta_t}$.
\end{enumerate}
Under these conditions,
\[
\mathbb E\!\left[\ell\!\big(f_t(x_t),y_t\big)\,\middle|\,\mathcal F_{t-1}\right]
=
\mathbb E_{(x,y)\sim p_{\theta_t}}\!\left[\ell\!\big(f_t(x),y\big)\right]
=:R(\theta_t,f_t),
\]
so the centered loss
\[
Z_t:=\ell\!\big(f_t(x_t),y_t\big)-R(\theta_t,f_t)
\]
is a martingale difference sequence with respect to $(\mathcal F_t)$. 

\subsection{Regularity and bounded geometry}
\label{subsec:regularity}

We impose standing assumptions ensuring that Fisher geometry is well behaved on the region visited by the realized trajectory.

\begin{assumption}[Local statistical regularity]
\label{assump:regularity}
The realized trajectory $\{\theta_t\}_{t=1}^{T+1}$ remains inside a compact set $K\subset\Theta$ on which:
\begin{enumerate}
\item \emph{Uniform Fisher regularity.} The Fisher information matrix $g(\theta)$ is finite, positive definite, and smooth on $K$.
\item \emph{Regular likelihood calculus.} The map $\theta\mapsto \log p_\theta(x,y)$ is twice continuously differentiable on a neighborhood of $K$, and differentiation and integration interchange uniformly on $K$.
\item \emph{Bounded geometry.} There exists an injectivity radius
$r_{\mathrm{inj}}>0$ such that for all $\theta\in K$,
$\operatorname{Exp}_\theta$ is a diffeomorphism on the $g_\theta$-ball of radius
$r_{\mathrm{inj}}$, and the sectional curvature on $K$ is uniformly bounded.
\item \emph{Sub-Gaussian sampling noise.} The martingale differences $Z_t$ defined above satisfy the conditional sub-Gaussian mgf bound
\[
\mathbb E\!\left[\exp(\lambda Z_t)\mid \mathcal F_{t-1}\right]
\le \exp(\tfrac12 \sigma^2 \lambda^2),
\qquad \text{for all }\lambda\in\mathbb R.
\]
\end{enumerate}
\end{assumption}

\begin{assumption}[Local dynamics regularity]
\label{assump:environment}
\mbox{}
\begin{enumerate}
\item \emph{Finite-energy actions.} The learner's control inputs satisfy
\[
  \mathbb{E}\|u_t\|^2 \le B,
  \qquad t=1,\dots,T.
\]

\item \emph{Local control differentiability of the environment map (chart form).}
Fix a local coordinate chart on $K$.
For each $\theta\in K$ and each realized $\eta_t$, the environment transition $u\mapsto F(\theta,u,\eta_t)$ is differentiable at $u=0$ and admits the expansion
\[
F(\theta_t,u_t,\eta_t)
=
F(\theta_t,0,\eta_t)
+ J_uF(\theta_t,0,\eta_t)\,u_t + r_t,
\]
where $r_t$ is a second-order chart remainder. Moreover, there exists $c>0$ uniform over $\theta_t\in K$ and admissible $u_t$ such that
\[
\|r_t\| \le c\,\|u_t\|^2
\]
in the chart Euclidean norm. By norm equivalence on compact $K$, we will always measure this remainder at the baseline next state, so that (possibly with a different constant)
\[
\|r_t\|_{g_{F(\theta_t,0,\eta_t)}} \;\lesssim\; \|u_t\|^2.
\]

\item \emph{Trajectory confinement.} The update $\theta_{t+1}=F(\theta_t,u_t,\eta_t)$
keeps $\theta_{t+1}\in K$ for all $t$.
\end{enumerate}
\end{assumption}

\begin{remark}[intrinsic interpretation of the remainder.]
When steps remain inside a normal neighborhood, the chart expansion in Assumption~\ref{assump:environment}(2) may be viewed as an intermediate representation of the intrinsic displacement from $F(\theta_t,0,\eta_t)$ to $F(\theta_t,u_t,\eta_t)$; conversions to Fisher--Rao distance are handled later using local bounded-geometry comparisons along the realized trajectory.
\end{remark}

\begin{lemma}[Uniform local KL quadratic bound on $K$]
\label{lem:kl_quadratic}
Under Assumption~\ref{assump:regularity}(1)--(3), there exist constants
$\rho\le r_{\mathrm{inj}}$ and $C_{\mathrm{KL}}<\infty$ such that for all
$\theta\in K$ and all $\Delta\theta\in T_\theta\Theta$ with
$\|\Delta\theta\|_{g_\theta}\le \rho$,
\[
D_{\mathrm{KL}}(p_{\operatorname{Exp}_\theta(\Delta\theta)}\,\Vert\,p_\theta)
\;\le\;
C_{\mathrm{KL}}\,\|\Delta\theta\|_{g_\theta}^2.
\]
\end{lemma}

\begin{lemma}[Local geodesic equivalence]
\label{lem:local_equivalence}
Let $(\Theta,g)$ be a statistical manifold equipped with the Fisher metric, and let $\theta\in K$. If \(v\in T_\theta\Theta\) satisfies \(\|v\|_{g_\theta}<r_{\mathrm{inj}}\), then the curve
\[
\gamma_v:[0,1]\to\Theta,\qquad
\gamma_v(s)=\operatorname{Exp}_\theta(sv),
\]
is the radial geodesic from \(\theta\) to \(\operatorname{Exp}_\theta(v)\). It is minimizing, and
\begin{equation}
\label{eq:local_equiv}
d_F\!\bigl(\theta,\operatorname{Exp}_\theta(v)\bigr)=\|v\|_{g_\theta}.
\end{equation}
\end{lemma}

\begin{proof}[Proof (see Appendix~\ref{app:proof_local_equivalence})]
Since $\|v\|_{g_\theta}\le r_{\mathrm{inj}}$, the radial geodesic is minimizing on $[0,1]$. Gauss's lemma implies its length equals $\|v\|_{g_\theta}$, yielding \eqref{eq:local_equiv}; see \citet[Prop.~6.10 and Cor. ~6.11]{Lee97}.
\end{proof}

\subsection{Concentration tools}
\label{subsec:concentration}

\begin{lemma}[Sub-Gaussian martingale deviation]
\label{lem:subgaussian_martingale}
Let $\{Z_t\}_{t=1}^T$ be a martingale difference sequence adapted to $\{\mathcal F_t\}$ such that
\[
\mathbb E\!\left[\exp(\lambda Z_t)\mid \mathcal F_{t-1}\right]
\le
\exp\!\left(\tfrac12 \sigma_t^2 \lambda^2\right)
\qquad \text{for all }\lambda\in\mathbb R.
\]
Set $S_T=\sum_{t=1}^T Z_t$ and $\mathcal{V}_T=\sum_{t=1}^T \sigma_t^2$.
Then for any $\eta>0$,
\[
\Pr\!\big(|S_T|\ge \eta\big)
\le
2\exp\!\left(-\frac{\eta^2}{2\mathcal{V}_T}\right).
\]
\end{lemma}

\begin{corollary}[Expected deviation]
\label{cor:subgaussian_expectation}
Under the conditions of Lemma~\ref{lem:subgaussian_martingale},
\[
\mathbb E\,|S_T|
\le
\sqrt{2\pi\,\mathcal{V}_T},
\qquad
\mathbb E\!\left|\frac{1}{T}\sum_{t=1}^T Z_t\right|
\le
\frac{\sqrt{2\pi\,\mathcal{V}_T}}{T}.
\]
In particular, if each $Z_t$ is $\sigma$--sub-Gaussian so that $\mathcal{V}_T\le \sigma^2T$,
then
\[
\mathbb E\!\left|\frac{1}{T}\sum_{t=1}^T Z_t\right|
\le
\frac{\sqrt{2\pi}\,\sigma}{\sqrt{T}}.
\]
\end{corollary}

\begin{proof}[Proof (see Appendix~\ref{app:proof_lem_subg_martingale_cor_subg_exp})]
Integrate the tail bound of Lemma~\ref{lem:subgaussian_martingale}:
\[
\mathbb E\,|S_T|
=\int_0^\infty \Pr(|S_T|\ge \eta)\,d\eta
\le
\int_0^\infty 2\exp\!\Big(-\frac{\eta^2}{2\mathcal{V}_T}\Big)\,d\eta
=
\sqrt{2\pi\,\mathcal{V}_T}.
\]
Divide by $T$ for the averaged bound.
\end{proof}

\subsection{Length identity for piecewise geodesic interpolation}
\label{subsec:length_convergence}

\begin{lemma}[Discrete--continuous Fisher--Rao length identity]
\label{lem:length-identity}
Let $\theta_1,\ldots,\theta_{T+1}\in\Theta$ be such that for each $t=1,\ldots,T$ there exists a minimizing Fisher--Rao geodesic connecting $\theta_t$ to $\theta_{t+1}$. Let $\gamma_T:[0,1]\to\Theta$ be the piecewise $C^1$ curve obtained by concatenating these minimizing geodesics, parameterized at constant speed on each subinterval $[(t-1)/T,t/T]$ so that $\gamma_T((t-1)/T)=\theta_t$ and $\gamma_T(t/T)=\theta_{t+1}$. Then
\begin{equation}
\label{eq:fr-length-identity}
\sum_{t=1}^{T} d_F(\theta_{t+1},\theta_t)
=
\int_0^1 \big\|\dot\gamma_T(s)\big\|_{g_{\gamma_T(s)}} \, ds,
\end{equation}
where $\dot\gamma_T$ exists for almost every $s\in[0,1]$.
\end{lemma}

\begin{proof}[Proof]
For each $t$, the restriction $\gamma_T|_{[(t-1)/T,t/T]}$ is a constant-speed minimizing geodesic from $\theta_t$ to $\theta_{t+1}$. Hence its Riemannian length equals $d_F(\theta_{t+1},\theta_t)$. Summing over $t=1,\ldots,T$ yields \eqref{eq:fr-length-identity}.
\end{proof}

Together, Lemmas~\ref{lem:local_equivalence}--\ref{lem:length-identity} provide the geometric backbone used in Sections~\ref{sec:setup}--\ref{sec:theory}, and Lemma~\ref{lem:subgaussian_martingale} and Corollary~\ref{cor:subgaussian_expectation} provide the concentration backbone used to control the sampling term.

\section{Problem Setup}
\label{sec:setup}

Having established the geometric and concentration tools in Section~\ref{sec:preliminaries}, we now instantiate them in a stochastic process that captures learning under \emph{endogenous drift}---a regime in which the learner’s own actions alter the distribution from which future observations are drawn. Our goal is to express this closed-loop evolution directly in the Fisher geometry defined earlier, so that distributional motion can be quantified in intrinsic terms.

We begin by specifying the coupled learner--environment dynamics introduced in Section~\ref{sec:intro}. As in Section~\ref{sec:preliminaries}, let $(\mathcal X,\mathcal Y)$ denote the input--output space and $\Theta$ the parameter manifold indexing the family of distributions $\{p_\theta(x,y)\}_{\theta\in\Theta}$. At each discrete time~$t\in\{1,\ldots,T\}$,
\begin{equation}
\label{eq:drift_model}
(x_t,y_t)\mid \mathcal F_{t-1}\sim p_{\theta_t},
\qquad
u_t\sim\pi_t(\mathcal F_{t-1}),
\qquad
\theta_{t+1}=F(\theta_t,u_t,\eta_t),
\end{equation}
where $\mathcal F_t=\sigma(x_1,y_1,u_1,\ldots,x_t,y_t,u_t)$ is the natural filtration generated by past observations and actions. The map~$F$ governs the environment’s evolution, $\pi_t$ is a (possibly history--dependent) policy, and $\eta_t$ represents exogenous latent influences. This is the model we analyze here.

\begin{definition}[Endogenously Drifting Process]
\label{def:endogenous-drift}
A learning process $\{(x_t,y_t,\theta_t,u_t)\}_{t=1}^T$ satisfies \emph{endogenous drift} if the transition $F(\theta_t,u_t,\eta_t)$ depends on the control~$u_t$, thereby coupling the evolution of $\theta_t$ to the policy~$\pi_t$. The variable~$\eta_t$ represents the \emph{exogenous influence} acting on the environment: a collection of factors, deterministic or stochastic, that evolve independently of~$u_t$ and contribute \emph{exogenous drift}.
\end{definition}

\begin{example}[Linear--Gaussian environment]
\label{ex:linear_gaussian}
A classical linear instance of~\eqref{eq:drift_model} is
\[
F(\theta_t,u_t,\eta_t)
~=~
\theta_t + A u_t + B \eta_t,
\]
where $A$ and $B$ determine sensitivity to the learner’s action $u_t$ and to exogenous influences $\eta_t$. The term $Au_t$ captures policy-induced motion (vanishing when $A=0$), while $B\eta_t$ captures background evolution (vanishing when $B=0$).

This instance satisfies Definition~\ref{def:endogenous-drift} exactly and yields closed-form expressions for Fisher--Rao motion, drift magnitudes $(d_t,\kappa_t^{(\mathcal M)})$, and the intrinsic budget $C_T$. It underlies the linear--Gaussian validation in Section~\ref{sec:experiments}.

\medskip
\noindent
\emph{Instantiation (linear regression).}
One concrete realization is linear regression with Gaussian covariates $x_t \sim \mathcal{N}(\mu_t,\Sigma)$ and noise $y_t = w^\top x_t + \varepsilon_t$, where $\varepsilon_t \sim \mathcal{N}(0,\sigma^2)$, with environment state $\theta_t=\mu_t$ evolving according to the same map $F$.
\end{example}

Endogenous drift thus produces a closed-loop system in which learner and environment co-evolve. If $F$ is independent of~$u_t$, the setting reduces to non-stationary optimization with purely exogenous drift. If $F$ is constant, it reduces to the classical iid\ regime. The general case---where both influences are present---yields an evolving distributional flow driven jointly by background change and policy-induced perturbations.

A related question is whether such coupling might guide the process toward a performative equilibrium, so that policy-induced drift \(\kappa^{(M)}_t\) eventually vanishes. In some settings this may occur, in which case the policy-sensitive contribution to one-step distributional motion correspondingly disappears. In general, however, we do not assume such convergence. The joint evolution of \((f_t,\theta_t)\) is not assumed to be a gradient flow on a shared potential, and the environment map \(F\) need not be contractive. Even when fixed points exist, stochasticity, misspecification, exogenous change, or misaligned objectives can yield persistent cycles or wandering finite-horizon trajectories. Prior work on performative prediction typically imposes stability conditions to study equilibria \citep{Perdomo2020}. Our analysis instead addresses the moving regime itself, in which the learner--environment system may continue to move along the statistical manifold and prequential step-to-step reproducibility is governed by the corresponding motion rate.

With the dynamics specified, we now express the relevant statistical quantities along the evolving trajectory. The parameter space $\Theta$ carries the Fisher--Rao metric $g_\theta$ and geodesic distance~$d_F$ introduced in Section~\ref{sec:preliminaries}. All norms, distances, and path lengths are understood with respect to this metric, with local equivalence and curvature bounds provided by the regularity conditions. We will keep two distinct effects separate. The first is the usual sampling fluctuation. Even if the data--generating law were fixed, the empirical trajectory average need not equal its conditional expectation. The second is geometric. As $\theta_t$ moves, the population risk of a \emph{within-step fixed} predictor can change from one step to the next.

We write the empirical trajectory average and corresponding population risk
\[
\widehat R_T(f,\pi)
=\frac{1}{T}\sum_{t=1}^T \ell\!\bigl(f_t(x_t),y_t\bigr),
\qquad
R(\theta,f):=\mathbb E_{(x,y)\sim p_\theta}\!\big[\ell(f(x),y)\big].
\]

\paragraph{Trajectory (same-time) risk and sampling deviation.}
The realized-trajectory population risk is
\[
R_T^{\mathrm{traj}}(f,\pi)
=\frac{1}{T}\sum_{t=1}^T R(\theta_t,f_t)
=\frac{1}{T}\sum_{t=1}^T
\mathbb{E}_{(x,y)\sim p_{\theta_t}}
\!\left[\ell(f_t(x),y)\right].
\]
Its sampling deviation is
\[
\Delta_T^{\mathrm{sam}}(f,\pi)
~=~
\bigl|\widehat R_T(f,\pi)-R_T^{\mathrm{traj}}(f,\pi)\bigr|.
\]
Under the filtration conditions stated in Section~\ref{sec:preliminaries}, the centered losses
\[
Z_t:=\ell\!\bigl(f_t(x_t),y_t\bigr)-R(\theta_t,f_t)
\]
form a martingale difference sequence, so $\Delta_T^{\mathrm{sam}}$ is controlled by martingale concentration (as in the stationary case).

\paragraph{Prequential risk and drift penalty.}
To quantify distributional motion without entangling it with the evolution of $f_t$, we measure how the population risk of the deployed predictor $f_t$ changes across the environment transition $\theta_t\to\theta_{t+1}$:
\[
v_t(f,\pi)
~:=~
\bigl|R(\theta_{t+1},f_t)-R(\theta_t,f_t)\bigr|.
\]
Its time average
\[
V_T(f,\pi)
~:=~
\frac{1}{T}\sum_{t=1}^T v_t(f,\pi)
\]
isolates the effect of environment motion: if $\theta_{t+1}=\theta_t$ then $V_T(f,\pi)=0$ regardless of how $f_t$ evolves.

The prequential (one-step-ahead) population risk is
\[
R_T^{+}(f,\pi)
~:=~
\frac{1}{T}\sum_{t=1}^T R(\theta_{t+1},f_t),
\]
and we define the corresponding prequential gap
\[
\Delta_T^{\mathrm{rep}}(f,\pi)
~:=~
\bigl|\widehat R_T(f,\pi)-R_T^{+}(f,\pi)\bigr|.
\]
An add--subtract decomposition yields the full separation
\[
\Delta_T^{\mathrm{rep}}(f,\pi)
\le
\underbrace{\left|\frac1T\sum_{t=1}^T
\bigl(\ell(f_t(x_t),y_t)-R(\theta_t,f_t)\bigr)\right|}_{\Delta_T^{\mathrm{sam}}(f,\pi)}
+
\underbrace{\frac1T\sum_{t=1}^{T}\bigl|R(\theta_{t+1},f_t)-R(\theta_t,f_t)\bigr|}_{V_T(f,\pi)}.
\]
Equivalently,
\begin{equation}
\label{eq:rep_decomp_setup}
\Delta_T^{\mathrm{rep}}(f,\pi)
\;\le\;
\Delta_T^{\mathrm{sam}}(f,\pi)
\;+\;
V_T(f,\pi).
\end{equation}

Thus sampling noise contributes the classical $T^{-1/2}$ term, while drift enters only through the geometric penalty $V_T$, which will be controlled by Fisher--Rao motion along the realized trajectory.

Equation \eqref{eq:rep_decomp_setup} also shows why standard average loss curves can be misleading in closed loop. Because $\Delta_T^{\mathrm{rep}}$ is a \emph{trajectory-average} gap, it can remain small even when substantial drift occurs, so long as the resulting risk shifts do not register as persistent average discrepancy. By contrast, $V_T$ records the \emph{within-step} population-risk change of a fixed predictor across each transition $\theta_t\!\to\!\theta_{t+1}$, and therefore isolates drift effects that may be invisible in aggregate averages. The subsequent analysis relates this drift penalty to intrinsic Fisher motion along the trajectory (via $C_T$), and Section~\ref{sec:theory} formalizes the associated speed-limit regime by matching the
upper bound with a lower bound for the same prequential reproducibility gap.

The drift primitives $(d_t,\kappa_t^{(\mathcal M)})$ and the budget $C_T$ are chosen precisely so that Fisher--Rao motion controls $V_T$. We therefore describe the environment’s motion on $(\Theta,g_\theta)$ by separating an exogenous baseline step from a policy--sensitive perturbation.

\paragraph{Exogenous drift.}
The exogenous drift magnitude is
\[
d_t
:= d_F\!\big(F(\theta_t,0,\eta_t),\,\theta_t\big),
\]
capturing the motion that would occur at time $t$ in the absence of control.

\paragraph{Policy-sensitive drift.}
Policy-sensitive drift is measured by the leading-order Fisher--metric motion induced by the learner’s action. Under Assumption~\ref{assump:environment}, the map
$u\mapsto F(\theta_t,u,\eta_t)$ is differentiable at $u=0$, so
\[
F(\theta_t,u_t,\eta_t)
=
F(\theta_t,0,\eta_t)
+
J_uF(\theta_t,0,\eta_t)\,u_t
+
r_t,
\]
with a quadratic remainder. The differential $J_uF(\theta_t,0,\eta_t)\,u_t$ is a tangent vector at the baseline next state $F(\theta_t,0,\eta_t)$, and we define
\begin{equation}
\label{eq:kappa-grad}
\kappa_t^{(\mathcal M)}
:=
\|J_uF(\theta_t,0,\eta_t)\,u_t\|_{g_{F(\theta_t,0,\eta_t)}}.
\end{equation}

We will use this linearization in Section~\ref{sec:drift_feedback} to compare $\kappa_t^{(\mathcal M)}$ to intrinsic Fisher--Rao displacements.

These components combine to form the decomposed Fisher--Rao drift budget
\[
C_T =
\sum_{t=1}^{T}
\left(d_t+\alpha \kappa_t^{(\mathcal M)}\right).
\]
The intrinsic path length \(\mathcal A_T\) records how far the realized law moves, while \(C_T\) records why it moves, separating exogenous and policy-sensitive contributions in the same local Fisher geometry. The scalar \(\alpha>0\) is not a model parameter. It is a fixed local comparison constant, uniform on \(K\), that converts the linearized policy-sensitive displacement into the same Fisher--Rao scale as the exogenous step. It may absorb constants from the bounded-geometry comparisons relating \(J_uF(\theta_t,0,\eta_t)u_t\) to the corresponding Fisher--Rao displacement.

The update~\eqref{eq:drift_model} defines a trajectory on $(\Theta,g_\theta)$. The quantities $d_t$ and $\kappa_t^{(\mathcal M)}$ describe its incremental motion, while $C_T$ records the accumulated budget for this motion. Together, $d_t$, $\kappa_t^{(\mathcal M)}$, and~$C_T$ serve as the geometric primitives of the model and anchor the drift--feedback bounds developed in Section~\ref{sec:theory}.

\section{Quantifying Drift and Feedback}
\label{sec:drift_feedback}

The quantities introduced in Section~\ref{sec:setup} describe the instantaneous motion of the data--generating process on the statistical manifold \((\Theta,g_\theta)\). This section shows how these local motions accumulate over time. The resulting drift--feedback budget controls the intrinsic Fisher--Rao path length up to a second-order residual and, through the Fisher--risk regularity condition, controls the trajectory-averaged population-risk variation. Combined with the sampling decomposition, these estimates lead to the prequential reproducibility bounds in Section~\ref{sec:theory}.

\subsection{Drift Decomposition}

Since $F$ is smooth in the control variable (Assumption~\ref{assump:environment}), we can expand the update in a fixed local coordinate chart $\Theta\subset\mathbb R^d$:
\begin{equation}
\label{eq:drift_decomp_clean}
\theta_{t+1}
~=~
F(\theta_t,0,\eta_t)
\;+\;
J_uF(\theta_t,0,\eta_t)\,u_t
\;+\;
r_t,
\end{equation}
where $r_t$ collects higher--order terms. The role of \eqref{eq:drift_decomp_clean} is purely \emph{comparative}: it is a coordinate representation used to control intrinsic Fisher--Rao distances on the compact region $K$ visited by the realized trajectory. The chart expansion is used only to obtain local comparison bounds; the
subsequent estimates are stated in Fisher--Rao distance.

To isolate the genuinely second--order contribution, define
\[
\epsilon_t := \|r_t\|_{g_{F(\theta_t,0,\eta_t)}}.
\]
By Assumption~\ref{assump:environment}(2), the remainder is pointwise quadratic, so
\begin{equation}
\label{eq:eps_bound}
\epsilon_t \le c \|u_t\|^2.
\end{equation}

The intrinsic one--step motion is measured by the Fisher--Rao distance $d_F(\theta_{t+1},\theta_t)$. We first split this motion into a baseline (exogenous) step and a policy--sensitive perturbation using the triangle inequality:
\begin{align}
\label{eq:triangle_split}
d_F(\theta_{t+1},\theta_t)
&\le
d_F\!\big(\theta_{t+1},F(\theta_t,0,\eta_t)\big)
\;+\;
d_F\!\big(F(\theta_t,0,\eta_t),\theta_t\big).
\end{align}
The second term is exactly the exogenous drift distance; recalling the definition from Section~\ref{sec:setup},
\[
d_t
:=d_F\!\big(F(\theta_t,0,\eta_t),\theta_t\big).
\]

The first term requires converting the local chart expansion in Assumption~\ref{assump:environment}(2) into an intrinsic distance comparison. By the bounded-geometry conditions in Assumption~\ref{assump:regularity}, and restricting to the local regime in which the controlled perturbation remains in a uniform normal neighborhood of \(F(\theta_t,0,\eta_t)\), the chart remainder can be absorbed into an intrinsic second-order term. Thus there is a constant \(\alpha>0\), uniform on \(K\), such that
\begin{equation}
\label{eq:local-control-distance}
d_F\!\left(\theta_{t+1},F(\theta_t,0,\eta_t)\right)
\leq
\alpha \kappa_t^{(\mathcal M)}+\epsilon_t,
\end{equation}
where
\[
\kappa_t^{(\mathcal M)}
=
\|J_uF(\theta_t,0,\eta_t)u_t\|_{g_{F(\theta_t,0,\eta_t)}}
\]
and $\epsilon_t$ is the second-order remainder controlled as in~\eqref{eq:eps_bound}. The constant $\alpha$ collects the uniform local comparison constants between the linearized controlled displacement in the chosen chart and Fisher--Rao distance
on $K$.

Combining \eqref{eq:triangle_split} and \eqref{eq:local-control-distance} yields:
\begin{equation}
\label{eq:one_step_motion}
d_F(\theta_{t+1},\theta_t)
\le 
d_t \;+\; \alpha\,\kappa_t^{(\mathcal M)} \;+\; \epsilon_t.
\end{equation}
Equation~\eqref{eq:one_step_motion} expresses intrinsic motion as the sum of the exogenous step, the policy sensitive contribution, and a second-order residual. The residual is isolated in $\epsilon_t$ and controlled by
\eqref{eq:eps_bound}.

\subsection{Cumulative Motion Rate and the Intrinsic Drift Budget}
\label{subsec:cumulative-motion}

Accumulating intrinsic one--step distances yields the discrete Fisher--Rao path length
\begin{equation}
\label{eq:path_length}
\mathcal{A}_T
~=~
\sum_{t=1}^T d_F(\theta_{t+1},\theta_t).
\end{equation}
Summing over $t$ gives
\begin{equation}
\label{eq:budget_relation}
\mathcal A_T \le \sum_{t=1}^T (d_t + \alpha \kappa_t^{(\mathcal M)} + \epsilon_t)
= C_T + \sum_{t=1}^T \epsilon_t.
\end{equation}
Recalling the decomposed Fisher--Rao drift budget
\[
C_T=\sum_{t=1}^T\bigl(d_t+\alpha\,\kappa_t^{(\mathcal M)}\bigr),
\]
Equation~\eqref{eq:budget_relation} shows that \(C_T\) controls the realized Fisher--Rao path length up to the quadratic residual \(\sum_{t=1}^T\epsilon_t\), which is upper-bounded by control energy via \eqref{eq:eps_bound}. Thus \(\mathcal A_T\) records the intrinsic length of the realized path, while \(C_T\) gives a decomposed budget for that motion, separating exogenous and policy-sensitive contributions in the same Fisher--Rao units.

\subsection{From Geometric Motion to Risk Variation}
\label{subsec:motion-to-risk}

We now relate Fisher--Rao motion to variation in population risk. On the regularity region $K$, we assume the  population risk is Lipschitz in Fisher--Rao distance. Lemma~\ref{lem:local_fisher_risk} in Section~\ref{subsec:main_theorem} gives a convenient sufficient condition under which this Lipschitz-in-Fisher property holds.

\begin{assumption}[Risk regularity on $K$]
\label{assump:risk_lipschitz}
There exists $L_p<\infty$ such that for all $\theta,\theta'\in K$ and all measurable predictors $f$ considered by the learner,
\[
|R(\theta',f)-R(\theta,f)|
~\le~
L_p\,d_F(\theta',\theta),
\]
where $R(\theta,f)$ is the population risk defined in Section~\ref{sec:setup}.
\end{assumption}

Applying Assumption~\ref{assump:risk_lipschitz} with $(\theta,\theta')=(\theta_t,\theta_{t+1})$ gives
\begin{equation}
\label{eq:risk_drift_bound}
|R(\theta_{t+1},f_t)-R(\theta_t,f_t)|
~\le~
L_p\,d_F(\theta_{t+1},\theta_t).
\end{equation}
Combining \eqref{eq:risk_drift_bound} with the one-step motion bound \eqref{eq:one_step_motion} yields the pathwise risk-variation control
\begin{equation}
\label{eq:risk_motion_bound_pre}
|R(\theta_{t+1},f_t)-R(\theta_t,f_t)|
\;\le\;
L_p\,d_t
\;+\;
L_p\alpha\,\kappa_t^{(\mathcal M)}
\;+\;
L_p\,\epsilon_t.
\end{equation}
Using \eqref{eq:eps_bound}, the residual term is controlled by control energy, so there exist constants \(C_1,C_2,C_3>0\), depending only on \(L_p\), \(\alpha\),
the remainder constant \(c\), and the norm-equivalence constants on \(K\), such that
\begin{equation}
\label{eq:risk_motion_bound}
|R(\theta_{t+1},f_t)-R(\theta_t,f_t)|
\;\le\;
C_1\,d_t
\;+\;
C_2\,\kappa_t^{(\mathcal M)}
\;+\;
C_3\,\|u_t\|^2.
\end{equation}

For later use, recall the trajectory risk variation
\begin{equation}\label{eq:risk_stability}
V_T(f,\pi)
:=
\frac1T\sum_{t=1}^{T}
      \bigl|R(\theta_{t+1},f_t)-R(\theta_t,f_t)\bigr|.
\end{equation}
The bound \eqref{eq:risk_motion_bound_pre} converts Fisher--Rao motion into one-step population-risk change. After summing over \(t\), it gives the drift term in the prequential decomposition \eqref{eq:rep_decomp_setup}.

\section{Theoretical Results}
\label{sec:theory}

The geometric structure introduced above yields sharp limits for learning under the endogenous drift model \eqref{eq:drift_model}. We control $\Delta_T^{\mathrm{sam}}$ by martingale concentration and we bound $V_T$ by Fisher--Rao motion, relating it to the decomposed drift budget $C_T$. This gives a prequential reproducibility bound of order \(T^{-1/2}+C_T/T\). We then prove a matching lower bound for the same gap \(\Delta_T^{\mathrm{rep}}\) on a canonical regular drift--feedback subclass, showing that this scaling is sharp for prequential reproducibility itself. A second, information-theoretic lower bound shows that this drift effect can remain statistically indistinguishable at the level of realized prequential evidence.

Throughout this section, expectations and suprema are taken over the relevant classes of drift--feedback processes generated by Assumptions~\ref{assump:regularity}--\ref{assump:risk_lipschitz} and the transition model \(\theta_{t+1}=F(\theta_t,u_t,\eta_t)\) in \eqref{eq:drift_model}, with the specified learning rules and policies. Full proofs are provided in Appendix~\ref{app:prelim_proofs}.

\subsection{Local Fisher--Risk Coupling}
\label{subsec:local_fisher_risk}

Let the population loss of the current learner \(f_t\) under environment
parameter \(\theta\) be
\[
\mathcal L_t(\theta)
~:=~
\mathbb E_{(x,y)\sim p_\theta}\!\left[\ell(f_t(x),y)\right].
\]
The sampling argument controls empirical fluctuation around \(R(\theta_t,f_t)\). To control the drift term, we also need population risk to vary regularly as the environment parameter moves. Assumption~\ref{assump:risk_lipschitz} imposes this as a uniform Fisher--Rao Lipschitz condition on \(K\), and Lemma~\ref{lem:local_fisher_risk} gives a convenient sufficient condition with an explicit constant.

\begin{lemma}[Local Fisher--Risk Coupling]
\label{lem:local_fisher_risk}
Fix $t$ and write $\mathcal L_t(\theta)=\mathbb E_{p_\theta}[\ell(f_t(x),y)]$.
Assume that for this $t$ the centered loss
\[
\widetilde Z_{\theta,t}(x,y)
:=
\ell(f_t(x),y)-\mathcal L_t(\theta)
\]
has a uniform second-moment envelope over \(K\):
\[
\sup_{\theta\in K}
\mathbb E_{p_\theta}\!\left[
\widetilde Z_{\theta,t}(x,y)^2
\right]
~\le~
\sigma_2^2 .
\]
Then Assumption~\ref{assump:risk_lipschitz} holds for this $t$ with
\(L_p=\sigma_2\).
In particular, the same conclusion holds under the stronger uniform
\(\sigma\)-sub-Gaussian condition with
\(L_p=\sqrt{c_{\mathrm{sg}}}\,\sigma\), where \(c_{\mathrm{sg}}\) is the universal constant relating sub-Gaussian tails to second moments under the adopted convention.
\end{lemma}

\begin{proof}[Proof]
Let $\gamma:[0,1]\to K$ be a constant-speed Fisher geodesic from $\theta$ to $\theta'$, and define $g(s)=\mathcal L_t(\gamma(s))$. Differentiating under the integral sign (justified by Assumption~\ref{assump:regularity}) yields
\[
g'(s)
=
\mathbb E_{(x,y)\sim p_{\gamma(s)}}\!\left[
\ell(f_t(x),y)\,
\langle s_{\gamma(s)}(x,y),\dot\gamma(s)\rangle
\right],
\]
where $s_\vartheta=\nabla_\vartheta\log p_\vartheta$ is the score. Using $\mathbb E_{p_{\vartheta}}[s_{\vartheta}]=0$ and centering the loss gives
\[
g'(s)
=
\mathbb E_{p_{\gamma(s)}}\!\left[
\widetilde Z_{\gamma(s),t}(x,y)\,
\langle s_{\gamma(s)}(x,y),\dot\gamma(s)\rangle
\right].
\]
By Cauchy--Schwarz and the definition of the Fisher metric,
\[
|g'(s)|
\le
\sqrt{\mathbb E_{p_{\gamma(s)}}\!\left[\widetilde Z_{\gamma(s),t}^2\right]}\,
\|\dot\gamma(s)\|_{g_{\gamma(s)}}.
\]
The uniform second-moment envelope implies
\[
\mathbb E_{p_{\gamma(s)}}\!\left[\widetilde Z_{\gamma(s),t}^2\right]
\le
\sigma_2^2,
\]
uniformly in \(s\). Hence
\[
|g'(s)|
\le
\sigma_2\,\|\dot\gamma(s)\|_{g_{\gamma(s)}} .
\]
Integrating along $\gamma$ yields
\[
|\mathcal L_t(\theta')-\mathcal L_t(\theta)|
\le
\sigma_2 \int_0^1 \|\dot\gamma(s)\|_{g_{\gamma(s)}}\,ds
=
\sigma_2\, d_F(\theta',\theta),
\]
which is exactly Assumption~\ref{assump:risk_lipschitz} with \(L_p=\sigma_2\).
The final sub-Gaussian statement follows because uniform \(\sigma\)-sub-Gaussianity implies
\[
\mathbb E_{p_\theta}\!\left[\widetilde Z_{\theta,t}^2\right]
\le
c_{\mathrm{sg}}\sigma^2 .
\]
\end{proof}

\paragraph{Interpretation.}
Assumption~\ref{assump:risk_lipschitz} is the geometry-to-risk regularity condition used to control \(V_T\). Lemma~\ref{lem:local_fisher_risk} gives a simple sufficient condition: a uniform second-moment envelope on the centered loss makes \(\mathcal L_t(\theta)\) Lipschitz in Fisher--Rao distance on \(K\). Uniform sub-Gaussianity is one convenient way to obtain such an envelope, but is stronger than the proof requires. This regularity is separate from dynamical stability; it does not require the learner--environment system to move slowly or converge.
\subsection{Drift--Feedback Bounds}
\label{subsec:main_theorem}

We now bound the prequential gap $\Delta_T^{\mathrm{rep}}$ by assembling the decomposition \eqref{eq:rep_decomp_setup}, martingale concentration for the sampling term, and Fisher--Rao motion control for the drift penalty $V_T$.

\begin{lemma}[Summed drift penalty]
\label{lem:VT_budget}
Let $V_T$ be as in \eqref{eq:risk_stability}. Under Assumption~\ref{assump:risk_lipschitz},
\begin{equation}\label{eq:VT_budget_pathwise}
V_T(f,\pi)
\;\le\;
\frac{L_p}{T}\Big(C_T + \sum_{t=1}^T \epsilon_t\Big),
\end{equation}
where $\epsilon_t$ is the remainder from \eqref{eq:eps_bound}. Consequently, using \eqref{eq:eps_bound},
\[
V_T(f,\pi)\;\le\;\frac{L_p}{T}\,C_T \;+\; \frac{L_pc}{T}\sum_{t=1}^T \|u_t\|^2 .
\]
\end{lemma}

\begin{proof}[Proof (See Appendix~\ref{app:proof_lem_VT_budget})]
By \eqref{eq:risk_drift_bound} and \eqref{eq:one_step_motion},
$|R(\theta_{t+1},f_t)-R(\theta_t,f_t)|
\le
L_p(d_t+\alpha\,\kappa_t^{(\mathcal M)}+\epsilon_t)$.
Sum over $t$ and divide by $T$ to obtain \eqref{eq:VT_budget_pathwise}, then apply \eqref{eq:eps_bound}.
\end{proof}

\begin{lemma}[Collected bounds]
\label{lem:collected_bounds}
For any learner $(f,\pi)$ in the above setting,
\[
\mathbb E\,\Delta_T^{\mathrm{sam}}(f,\pi)
\;\le\;
\frac{\sqrt{2\pi}\,\sigma}{\sqrt{T}},
\qquad
\mathbb E\,V_T(f,\pi)
\;\le\;
\frac{L_p}{T}\,\mathbb E[C_T]
+
\frac{L_p c}{T}\sum_{t=1}^T \mathbb E\|u_t\|^2 .
\]
\end{lemma}

\begin{proof}[Proof]
For the sampling term, $Z_t=\ell(f_t(x_t),y_t)-R(\theta_t,f_t)$ is a martingale difference sequence under \eqref{eq:drift_model}, so Corollary~\ref{cor:subgaussian_expectation} gives $\mathbb E\,\Delta_T^{\mathrm{sam}}\le \sqrt{2\pi}\sigma/\sqrt{T}$.
For the drift term, apply Lemma~\ref{lem:VT_budget} and take expectations.
\end{proof}

\begin{theorem}[Prequential reproducibility under drift]
\label{thm:main_rep}
For any learner $(f,\pi)$ in the above setting,
\[
\mathbb E\,\Delta_T^{\mathrm{rep}}(f,\pi)
\;\le\;
\frac{\sqrt{2\pi}\,\sigma}{\sqrt T}
\;+\;
\frac{L_p}{T}\,\mathbb E[C_T]
\;+\;
\frac{L_p c}{T}\sum_{t=1}^T \mathbb E\|u_t\|^2.
\]
\end{theorem}

\begin{proof}[Proof (See Appendix~\ref{app:proof_thm_main_rep})]
We start from the pathwise decomposition \eqref{eq:rep_decomp_setup}:
\[
\Delta_T^{\mathrm{rep}}(f,\pi)
\le
\Delta_T^{\mathrm{sam}}(f,\pi)+V_T(f,\pi).
\]
Taking expectations yields
\begin{equation}
\label{eq:rep_expected_bound}
\mathbb E\,\Delta_T^{\mathrm{rep}}(f,\pi)
\le
\mathbb E\,\Delta_T^{\mathrm{sam}}(f,\pi)
+
\mathbb E\,V_T(f,\pi).
\end{equation}

Define $Z_t:=\ell(f_t(x_t),y_t)-R(\theta_t,f_t)$. Under \eqref{eq:drift_model}, $\{Z_t\}_{t=1}^T$ is a martingale difference sequence:
\[
\mathbb E[Z_t\mid\mathcal F_{t-1}]=0.
\]
Assumption~\ref{assump:regularity}(4) gives the conditional $\sigma$--sub-Gaussian mgf bound.
Hence Corollary~\ref{cor:subgaussian_expectation} yields
\[
\mathbb E\,\Delta_T^{\mathrm{sam}}(f,\pi)
=
\mathbb E\left|\frac1T\sum_{t=1}^T Z_t\right|
\le
\frac{\sqrt{2\pi}\,\sigma}{\sqrt{T}}.
\]

By Lemma~\ref{lem:VT_budget} and \eqref{eq:eps_bound},
\[
V_T(f,\pi)
\le
\frac{L_p}{T}\,C_T
+
\frac{L_pc}{T}\sum_{t=1}^T \|u_t\|^2.
\]
Taking expectations yields
\[
\mathbb E\,V_T(f,\pi)
\le
\frac{L_p}{T}\,\mathbb E[C_T]
+
\frac{L_pc}{T}\sum_{t=1}^T \mathbb E\|u_t\|^2.
\]

Substituting the two displays into \eqref{eq:rep_expected_bound} yields the theorem.
\end{proof}

\paragraph{Budgeted drift--feedback classes.}
For \(C>0\) and \(\Lambda<\infty\), let \(\mathcal P_C(\Lambda)\) denote the class of regular drift--feedback processes satisfying Assumptions~\ref{assump:regularity}--\ref{assump:risk_lipschitz} and the local comparison conditions above, with
\[
\mathbb E_P C_T \le C,
\qquad
\mathbb E_P\sum_{t=1}^T \epsilon_t \le \Lambda C .
\]
The first condition bounds the leading Fisher--Rao drift budget. The second keeps
the accumulated local-comparison remainder on the same order as that budget, so
that the rate statement is governed by \(C/T\) rather than by an uncontrolled
second-order term. Under the quadratic remainder bound
\(\epsilon_t\le c\|u_t\|^2\), the second condition is implied by accumulated
control energy of order \(C\). When \(\Lambda\) is fixed, we write
\(\mathcal P_C\) and absorb constants depending on \(\Lambda\) into the comparison notation.

\begin{corollary}[Budgeted-rate upper bound]
\label{cor:budgeted_upper}
On any budgeted drift--feedback class \(\mathcal P_C\),
\[
\sup_{P\in\mathcal P_C}
\mathbb E_P\Delta_T^{\mathrm{rep}}
\lesssim
T^{-1/2}+\frac{C}{T},
\]
where the implicit constant depends only on the fixed regularity, risk-Lipschitz,
local-comparison, and \(\Lambda\)-constants.
\end{corollary}

\begin{proof}
Apply Theorem~\ref{thm:main_rep} and use
\(\mathbb E_P C_T\le C\) and
\(\mathbb E_P\sum_t\epsilon_t\le \Lambda C\).
\end{proof}
\paragraph{Interpretation.}
Theorem~\ref{thm:main_rep} separates the prequential gap into a classical sampling term scaling as $T^{-1/2}$ and a drift penalty controlled by the average intrinsic budget $C_T/T$. Thus, when the drift term is small relative to \(T^{-1/2}\), the bound is sampling-dominated; when the Fisher-motion term induces risk variation at a larger scale, the drift component dominates the upper bound.

\subsection{Sharpness of the Gap and Indistinguishability of Drift Effects}
\label{subsec:lower_bound}

We now give two complementary lower bounds. The first shows that the rate in Theorem~\ref{thm:main_rep} is sharp for the same prequential reproducibility gap \(\Delta_T^{\mathrm{rep}}\). The second shows that order-\(C/T\) effects on the one-step-ahead target need not be identifiable from the realized prequential evidence alone. Together, these results show that the \(C/T\) term is not an artifact of the upper-bound argument. It can appear in the empirical prequential gap itself, and in hard instances the corresponding drift effect can be statistically difficult to distinguish.

\subsubsection{Sharpness for the prequential gap}

\begin{theorem}[Sharpness of the prequential reproducibility rate]
\label{thm:lower_bound}
There exist constants \(c,r_0>0\) and a canonical regular drift--feedback
subclass \(\mathcal P_C^{\mathrm{hard}}\subseteq\mathcal P_C\)
satisfying Assumptions~\ref{assump:regularity}--\ref{assump:risk_lipschitz}, with \(C_T\le C\), vanishing second-order remainder, and \(C/T\le r_0\), such that
\[
\sup_{P\in\mathcal P_C^{\mathrm{hard}}}
\mathbb E_P\,\Delta_T^{\mathrm{rep}}
\ge
c\max\left(T^{-1/2},\frac{C}{T}\right).
\]
\end{theorem}

\begin{proof}[Proof (see Appendix~\ref{app:proof_sharpness})]
Work in a one-dimensional regular exponential family restricted to a compact Fisher normal neighborhood and use Fisher arclength coordinates, so
\(d_F(\theta,\theta')=|\theta-\theta'|\). Let
\[
\theta_{t+1}=\theta_t+u_t,\qquad \eta_t\equiv0 .
\]
Then \(d_t\equiv0\), \(\kappa_t^{(\mathcal M)}=|u_t|\), and the second-order remainder vanishes.

For the drift component, choose an alternating local trajectory with active steps \(|u_t|=\delta\), remaining inside the compact neighborhood, and choose
\(\delta\) so that
\[
C_T=\alpha\sum_t |u_t|\asymp C .
\]
The bounded loss and history-independent predictor sequence in Appendix~\ref{app:proof_sharpness} align the local population-risk slope with the chosen motion, so that on active steps
\[
R(\theta_{t+1},f_t)-R(\theta_t,f_t)
\ge c_0 d_F(\theta_{t+1},\theta_t).
\]
Therefore
\[
R_T^+-R_T^{\mathrm{traj}}
=
\frac1T\sum_{t=1}^T
\{R(\theta_{t+1},f_t)-R(\theta_t,f_t)\}
\ge
\frac{c_0}{T}\sum_t d_F(\theta_{t+1},\theta_t)
\asymp \frac{C}{T}.
\]
Since \(\mathbb E_P[\widehat R_T-R_T^{\mathrm{traj}}]=0\), Jensen's inequality gives
\[
\mathbb E_P\Delta_T^{\mathrm{rep}}
=
\mathbb E_P|\widehat R_T-R_T^+|
\ge
\left|\mathbb E_P[\widehat R_T-R_T^+]\right|
=
|R_T^{\mathrm{traj}}-R_T^+|
\gtrsim \frac{C}{T}.
\]
The \(T^{-1/2}\) term follows from a stationary nondegenerate subclass
\((C_T=0)\), where \(\Delta_T^{\mathrm{rep}}=\Delta_T^{\mathrm{sam}}\) and ordinary sampling fluctuation has expectation \(\gtrsim T^{-1/2}\). Taking \(\mathcal P_C^{\mathrm{hard}}\) to be the union of the drift-bias subclass and this stationary subclass gives the displayed \(\max(T^{-1/2},C/T)\) lower bound.
\end{proof}

\paragraph{Interpretation.}
Theorem~\ref{thm:lower_bound} lower-bounds the same object controlled in Theorem~\ref{thm:main_rep}. The construction shows that the drift term can appear as a signed bias between realized-trajectory risk and the one-step-ahead population target. When the local risk shifts align with the Fisher motion, this bias accumulates at rate \(C/T\). Thus the \(C/T\) term is not an artifact of the upper-bound proof; it is forced by regular drift--feedback trajectories. The one-dimensional construction should be read as a hard Fisher direction inside
the regular class; embedding the same arclength submodel in a higher-dimensional
family gives the same obstruction.

\begin{corollary}[Sharp prequential speed limit]
\label{cor:speed_limit}
Let \(\mathcal P_C\) be a budgeted regular drift--feedback class with controlled second-order remainder that contains the canonical hard subclass
\(\mathcal P_C^{\mathrm{hard}}\) of Theorem~\ref{thm:lower_bound}. Then
\[
\sup_{P\in\mathcal P_C}
\mathbb E_P\Delta_T^{\mathrm{rep}}
=
\Theta\!\left(T^{-1/2}+\frac{C}{T}\right),
\]
with constants depending only on the fixed regularity, risk-Lipschitz, and remainder-control constants.
\end{corollary}

\paragraph{Interpretation.}
Together, the upper bound and Theorem~\ref{thm:lower_bound} show that \(T^{-1/2}+C/T\) is the sharp rate for prequential reproducibility itself. The average Fisher--Rao drift rate \(C/T\) is therefore the correct order of the drift contribution in this local regularity class.

\subsubsection{Information-theoretic indistinguishability of drift effects}

The preceding result shows that the \(C/T\) term is structurally sharp for \(\Delta_T^{\mathrm{rep}}\). We now give a complementary lower bound showing that the drift effect responsible for this term need not be identifiable from the realized prequential evidence. The construction uses short Fisher geodesic excursions that return to a common base point. Fisher--Rao path length, and hence the induced one-step-ahead risk variation, accumulates linearly in the excursion size, whereas KL divergence between candidate trajectory laws accumulates only quadratically. Thus a drift effect of order \(C/T\) can be present while the candidate processes remain hard to distinguish.

\begin{theorem}[Information-theoretic indistinguishability of drift effects]
\label{thm:indistinguishability}
There exists a canonical regular drift--feedback subclass
\(\mathcal P_C^{\mathrm{hide}}\) satisfying Assumptions~\ref{assump:regularity}--\ref{assump:risk_lipschitz} and
\(E_P C_T\le C\) such that, with
\[
R_T^+(P):=\frac1T\sum_{t=1}^T R(\theta_{t+1},f_t),
\]
\[
\inf_{\widehat R_T}
\sup_{P\in\mathcal P_C^{\mathrm{hide}}}
E_P\left|\widehat R_T-R_T^+(P)\right|
\ge
c\max\left(T^{-1/2},\frac{C}{T}\right).
\]
\end{theorem}

\begin{proof}[Proof (see Appendix~\ref{app:proof_invisibility})]
We define $\mathcal P_C^{\mathrm{hide}}$ by an explicit one-dimensional exponential-family model and a family of deterministic drift trajectories indexed by a codebook.

\emph{Model and geometry.}
Fix a one-dimensional regular exponential family $\{p_\theta\}$ on a compact interval where Fisher information is bounded above and below, and reparameterize into Fisher arclength coordinates so that $d_F(\theta,\theta')=|\theta-\theta'|$ on this interval \citep[cf.][]{AmariNagaoka2000}.
Consider controllable dynamics
\[
F(\theta_t,u_t,\eta_t)=\theta_t+u_t,\qquad \eta_t\equiv 0,
\]
so $d_t\equiv 0$, $J_uF\equiv 1$, and the remainder term vanishes.
In this subclass $\kappa_t^{(\mathcal M)}=|u_t|$ and $C_T=\alpha\sum_{t=1}^T |u_t|$.

\emph{Excursions and budget.}
Fix a step size $\delta>0$ and partition the horizon into $m$ disjoint two-step blocks. For a sign vector $v\in\{-1,+1\}^m$, define a history-independent policy that in block $j$ applies $u=v_j\delta$ and then $u=-v_j\delta$, returning to the base point after two steps. The resulting parameter path executes $m$ geodesic excursions of length $2\delta$, hence
\[
C_T(P_v)=\alpha\sum_{t=1}^T|u_t|=2\alpha m\delta.
\]
Choose
\[
m=\Big\lfloor \frac{C}{2\alpha\delta}\Big\rfloor \wedge \Big\lfloor \frac{T}{2}\Big\rfloor,
\]
so $C_T(P_v)\le C$ for all $v$ (in fact $C_T$ is deterministic here).

\emph{Risk separation.}
Fix a bounded loss and model pair and fix a (history-independent) learner sequence $(f_t)$ for which the population risk has a nonzero local slope at the base point $\theta_0$ on the indices used in the construction: there exists $c_0>0$ such that for all sufficiently small $\delta$,
\[
\bigl|R(\theta_0+\delta,f_{2j-1})-R(\theta_0-\delta,f_{2j-1})\bigr|
\;\ge\; c_0\,\delta,
\qquad j=1,\dots,m.
\]
If $v$ and $w$ differ in $\Omega(m)$ blocks, then on the corresponding excursions their one-step-ahead states $\theta_{2j}^{(v)}$ and $\theta_{2j}^{(w)}$ differ by $2\delta$ on $\Omega(m)$ indices, forcing
\[
\bigl|R_T^{+}(P_v)-R_T^{+}(P_w)\bigr|
\gtrsim \frac{m\delta}{T}
=
\frac{1}{T}\Big(\Big\lfloor \frac{C}{2\alpha\delta}\Big\rfloor \wedge \Big\lfloor \frac{T}{2}\Big\rfloor\Big)\delta
\;\asymp\; \min\!\Big(\frac{C}{T},\,\delta\Big),
\qquad
R_T^{+}(P):=\frac1T\sum_{t=1}^{T} R(\theta_{t+1},f_t).
\]
In particular, choosing $\delta$ so that $\lfloor C/(2\alpha\delta)\rfloor \le \lfloor T/2\rfloor$ yields separation $\gtrsim C/T$.

\emph{KL control and Fano.}
Because the trajectories are deterministic and the observations are conditionally independent given $\theta_t$, the joint law factorizes and
\[
D_{\mathrm{KL}}(P_v\Vert P_w)
=
\sum_{t=1}^T D_{\mathrm{KL}}\!\bigl(p_{\theta_t^{(v)}}\Vert p_{\theta_t^{(w)}}\bigr).
\]
The trajectories differ only on the indices in disagreeing blocks. On a compact interval, regular
exponential families satisfy a uniform local quadratic KL bound, so each disagreeing block
contributes \(O(\delta^2)\) KL on the step where the parameters differ. Thus
\(D_{\mathrm{KL}}(P_v\Vert P_w)\lesssim m\delta^2\).
Choose a binary codebook \(V\subset\{-1,+1\}^m\) with \(|V|\ge 2^{\Omega(m)}\) and minimum Hamming
distance \(\Omega(m)\) \citep[Lemma~2.9]{Tsybakov2009}. Taking \(\delta\) small enough so that the
resulting family \(\{P_v\}_{v\in V}\) satisfies the standard many-hypothesis Fano testing
condition, the usual testing-to-estimation reduction yields a minimax lower bound of order
\(m\delta/T\asymp C/T\) for estimating \(R_T^+(P)\) over \(\mathcal P_C^{\mathrm{hide}}\)
\citep[Theorem~2.7 and Sec.~2.7.1]{Tsybakov2009}; cf. \citet{Yu1997}. The Fano construction gives the \(C/T\) term. The \(T^{-1/2}\) term is obtained
separately by a standard two-point Le Cam argument on a stationary subclass
\((C_T=0)\). Taking the union of the Fano drift subclass and the stationary
subclass gives the combined \(\max(T^{-1/2},C/T)\) rate.
\end{proof}

\paragraph{Interpretation.}
Theorem~\ref{thm:indistinguishability} complements the sharpness result above. Theorem~\ref{thm:lower_bound} shows that an order-\(C/T\) term can appear in \(\Delta_T^{\mathrm{rep}}\) itself. The present theorem shows that, on a related hard subclass, the corresponding one-step-ahead effect need not be identifiable from the realized prequential evidence alone. Distinct regular drift trajectories can induce one-step-ahead targets separated at order \(C/T\), while producing loss/evidence streams that are statistically difficult to distinguish.

This is not a claim that drift is never observable, or that membership in a hard class is universally untestable. Rather, it shows that the realized stream may not contain enough information to determine which drift path, and hence which one-step-ahead target, generated the data. Section~\ref{sec:observable_footprints} therefore studies the separate question of which monitoring channels make risk-relevant drift directions visible in practice. 

\subsection{Stationary and Classical Limits}
\label{subsec:reductions}

Our results recover standard learning regimes as special cases of $(d_t,\kappa_t^{(\mathcal M)})$ by nullifying one or both components of geometric motion. We emphasize the separation established above. Sampling deviation around the realized trajectory risk is controlled at the classical $T^{-1/2}$ rate, while drift enters the prequential gap through the penalty term $V_T$ in $\Delta_T^{\mathrm{rep}}\le \Delta_T^{\mathrm{sam}}+V_T$.

\begin{itemize}
\item \textbf{iid\ regime.}
If $d_t\equiv 0$ and $J_uF\equiv 0$ (so $\theta_{t+1}=\theta_t$), then the environment does not move and $V_T(f,\pi)=0$ and $\Delta_T^{\mathrm{rep}}(f,\pi)=\Delta_T^{\mathrm{sam}}(f,\pi)$, giving $\mathbb E\,\Delta_T^{\mathrm{rep}}(f,\pi)=O(T^{-1/2})$, matching classical results \citep{Vapnik1998}.

\item \textbf{Exogenous drift.}
If $J_uF\equiv 0$ but $d_t>0$, then $\kappa_t^{(\mathcal M)}\equiv 0$ and Lemma~\ref{lem:collected_bounds} yields $\mathbb E\,V_T(f,\pi)\lesssim (1/T)\sum_t \mathbb E[d_t]$, which is directly analogous in form to variation--budget penalties in nonstationary optimization \citep{Besbes2015}.

\item \textbf{Performative equilibrium.}
If both drift components vanish asymptotically (so $d_t\to 0$ and $\kappa_t^{(\mathcal M)}\to 0$), then $\mathbb E\,V_T(f,\pi)\to 0$ and the only remaining contribution is sampling deviation at rate $T^{-1/2}$. This corresponds to regimes where the coupled learner--environment system stabilizes, as in performative equilibrium analyses \citep{Perdomo2020}.

\item \textbf{Feedback-driven drift.}
If $d_t\equiv 0$ but $\kappa_t^{(\mathcal M)}>0$, then $\mathbb E\,V_T(f,\pi)\lesssim (1/T)\sum_t \mathbb E[\kappa_t^{(\mathcal M)}]$. This captures purely policy-induced instability of the data-generating process and is conceptually aligned with feedback effects studied in adaptive data analysis, though the mechanism here is an evolving population law rather than a fixed database \citep{Dwork2015,russo2016,bassily2016}.
\end{itemize}

Formally, these regimes correspond to projections of the same geometric decomposition onto axes of the drift vector $(d_t,\kappa_t^{(\mathcal M)})$.
Setting one component to zero restricts the intrinsic drift budget $C_T=\sum_t(d_t+\alpha\,\kappa_t^{(\mathcal M)})$ to the remaining axis and yields the corresponding specialization of the stability bound.

Taken together, these results identify a geometric quantity that governs learning under drift. The intrinsic drift budget $C_T$ upper bounds the total Fisher--Rao motion of the coupled learner--environment system and controls drift-induced instability through the scaling $O(C_T/T)$, while sampling deviation retains the classical $O(T^{-1/2})$ rate.

\subsection{Practical implications}

The bounds imply that realized loss curves are not self-explanatory in closed-loop systems. The same empirical trajectory can be consistent with a slow environment, a fast exogenous drift path, or policy-induced amplification through the feedback loop. What distinguishes these regimes is not the loss sequence alone but the average information speed of the data-generating law, together with its decomposition into exogenous and policy-sensitive motion. Thus $C_T/T$ identifies when prequential performance is estimation-limited and when it is drift-limited, while $(d_t,\kappa_t^{(\mathcal M)})$ separates background motion from learner-induced instability.

\section{Observable Fisher Motion}
\label{sec:observable_footprints}

Our analysis above is phrased in intrinsic terms. The realized Fisher--Rao motion
\[
\mathcal A_T ~=~ \sum_{t=1}^{T} d_F\!\bigl(p_{\theta_{t+1}},p_{\theta_t}\bigr)
\]
is the path length of the realized environment trajectory on the statistical manifold of data-generating laws, and the drift budget $C_T$ is the decomposed drift budget that controls this motion through the local comparison in Section~\ref{sec:drift_feedback}. These quantities are appealing precisely because they are intrinsic to the family $\{p_\theta\}_{\theta\in\Theta}$ and do not depend on a particular coordinate choice for $\theta$. At the same time, they are defined at the level of the full law $p_{\theta_t}$, whereas an adaptive system typically observes only a coarsened view of each sample. 

Although we state the theory on a finite-dimensional statistical manifold $\{p_\theta\}_{\theta\in\Theta}$, the role of Fisher--Rao geometry is not inherently finite-dimensional. On a regular dominated family of densities one may define the Fisher--Rao metric directly by
\[
    g_p(h_1,h_2)
    =
    \int \frac{h_1(z)h_2(z)}{p(z)}\,dz,
\]
for tangent perturbations $h_i$ with zero integral, whenever the expression is finite. The finite-dimensional parameterization is used here to keep the drift--feedback decomposition explicit and to avoid additional functional analytic conditions. The observable-channel construction below does not require the full intrinsic law to be tractably represented, but defines a Fisher geometry on the pushed-forward laws available to the monitor.

To connect intrinsic drift to what can be monitored, fix a time homogeneous Markov kernel $K$ from the data space to an observation space $\mathcal O$. It may represent a deterministic statistic $o=\psi(z)$, a randomized coarsening (e.g., quantization with noise), or more generally any fixed measurable channel producing an observed representation of $z$. Applying $K$ to a law $p$ produces the pushforward distribution $K_\# p$ on $\mathcal O$,
\[
(K_\# p)(A)
~=~
\int K(A\mid z)\,p(dz),
\qquad
A\subseteq \mathcal O \ \text{measurable}.
\]
Given the family $\{p_\theta\}$, the channel induces a corresponding family on $\mathcal O$ via $q_\theta := K_\# p_\theta$. We may therefore speak of Fisher--Rao geometry not only for the intrinsic law but also for its channel image.

Fisher geometry is monotone under such channels. In \v{C}encov's characterization, Fisher information is the unique Riemannian metric that is invariant under sufficient statistics and, more generally, contracts under Markov morphisms \citep{cencov1982,AmariNagaoka2000}. In the present setting, this gives a concrete contraction statement for Fisher--Rao distances.

\begin{proposition}[Fisher--Rao contraction under Markov kernels]
\label{prop:fr_contraction_kernel}
Let $\{p_{\theta}\}_{\theta\in\Theta}$ be a regular statistical family equipped with Fisher information metric and associated geodesic distance $d_F$. Let $K$ be a time-homogeneous Markov kernel mapping the data space to $\mathcal O$, and define $q_{\theta}:=K_\# p_{\theta}$.
Equip $\{q_{\theta}\}$ with its Fisher--Rao geodesic distance. Then for all $\theta,\theta'\in\Theta$,
\begin{equation}
\label{eq:fr_contraction_kernel}
d_F\!\bigl(q_{\theta},q_{\theta'}\bigr)
\;\le\;
d_F\!\bigl(p_{\theta},p_{\theta'}\bigr).
\end{equation}
Consequently, for any realized trajectory $\theta_1,\ldots,\theta_{T+1}$,
\begin{equation}
\label{eq:path_length_contraction_kernel}
\sum_{t=1}^{T} d_F\!\bigl(q_{\theta_{t+1}},q_{\theta_t}\bigr)
\;\le\;
\sum_{t=1}^{T} d_F\!\bigl(p_{\theta_{t+1}},p_{\theta_t}\bigr)
~=~
\mathcal A_T.
\end{equation}
\end{proposition}

\begin{proof}
Monotonicity of the Fisher information metric under Markov morphisms \citep{cencov1982,AmariNagaoka2000} implies that the tangent map induced by $p\mapsto K_\# p$ contracts Fisher norms. Writing $g^{(p)}$ and $g^{(q)}$ for the Fisher information metrics on $\{p_\theta\}$ and $\{q_\theta\}$, respectively, for any tangent vector $v$ at $p_\theta$,
\[
\|K_\# v\|_{g^{(q)}_\theta}
\;\le\;
\|v\|_{g^{(p)}_\theta}.
\]
Therefore, the Fisher--Rao length of any smooth curve in $\Theta$ cannot increase after applying $K_\#(\cdot)$. Applying this to a minimizing Fisher--Rao geodesic between $\theta$ and $\theta'$ yields \eqref{eq:fr_contraction_kernel}. Summing along the realized trajectory gives \eqref{eq:path_length_contraction_kernel}.
\end{proof}

Proposition~\ref{prop:fr_contraction_kernel} turns a fixed observation channel into an intrinsic quantity on the output space. Along the trajectory, define $q_t := K_\# p_{\theta_t}$ and the observed Fisher motion
\[
\mathcal A_T^{(K)} ~=~ \sum_{t=1}^{T} d_F(q_{t+1},q_t),
\qquad
\text{with observed Fisher rate }~\mathcal A_T^{(K)}/T.
\]
This quantity depends only on what is visible through $\mathcal O$. It is computable when the induced family is tractable, and it is estimable from windowed samples when $\mathcal O$ is low-dimensional or well-approximated by a simple parametric family. It need not recover the intrinsic drift rate. Rather, it is the Fisher rate of the pushed-forward process and can only decrease as the observation becomes coarser. This distinction is important operationally: even when the full data-generating process is too high-dimensional, implicit, or irregular for the intrinsic finite-dimensional model to be used directly, a fixed channel may induce a regular pushed-forward process on which Fisher motion is still well defined.

Combining Proposition~\ref{prop:fr_contraction_kernel} with the intrinsic bound $\mathcal A_T \le C_T+\sum_t\epsilon_t$ yields an immediate link between observable rates and the drift budget.

\begin{corollary}[Observed Fisher motion is bounded by the intrinsic drift budget]
\label{cor:observable_footprint_budget}
Fix a Markov kernel $K$ and define $q_t := K_\# p_{\theta_t}$ along the trajectory. Let
\[
\mathcal A_T^{(K)}
~:=~
\sum_{t=1}^{T} d_F(q_{t+1},q_t).
\]
Then $\mathcal A_T^{(K)} \le \mathcal A_T$. Moreover, under the drift--feedback decomposition of Section~\ref{sec:drift_feedback},
\begin{equation}
\label{eq:observable_footprint_budget}
\frac{\mathcal A_T^{(K)}}{T}
\;\le\;
\frac{\mathcal A_T}{T}
\;\le\;
\frac{C_T}{T} + \frac{1}{T}\sum_{t=1}^T \epsilon_t,
\end{equation}
where $C_T$ is the intrinsic drift budget and $\epsilon_t$ is the second-order remainder term from \eqref{eq:one_step_motion}--\eqref{eq:budget_relation}.
\end{corollary}

\begin{proof}
The inequality $\mathcal A_T^{(K)}\le \mathcal A_T$ follows from \eqref{eq:path_length_contraction_kernel}. The bound on $\mathcal A_T$ is \eqref{eq:budget_relation}. Dividing by $T$ gives \eqref{eq:observable_footprint_budget}.
\end{proof}

The content of Corollary~\ref{cor:observable_footprint_budget} is straightforward. Any fixed channel induces an information speed on its output space. This speed is measured in Fisher units, is additive over time, and is dominated by the intrinsic drift rate up to the second-order term. A large observed Fisher rate therefore provides evidence that the realized trajectory is not moving slowly in the intrinsic geometry, whereas a small observed rate indicates only that the chosen channel does not reveal substantial Fisher motion on the timescale being monitored.

Finally, our emphasis on Fisher--Rao motion should not be read as a rejection of drift diagnostics based on standard divergences. Many monitoring approaches quantify change via KL, total variation, Wasserstein distance, or related scores between successive windows of model outputs, learned representations, features, or error statistics. Such diagnostics are useful generic change detectors, but they usually do not specify which representation preserves the directions of motion that matter for one-step risk. The present theory gives a geometric interpretation of this issue. When changes are incremental, KL and other smooth divergences admit local quadratic approximations governed by Fisher information, so divergence-based monitors can track projected Fisher motion when the monitored representation is well chosen. At the same time, this connection is representation-dependent. Under any fixed coarse-graining the induced divergences can only decrease, so an observed divergence that remains small may reflect limited sensitivity of the monitored representation rather than genuine stationarity.

The empirical question is therefore not only whether a monitored representation changes, or whether a channel contracts Fisher motion, but whether the contracted motion remains aligned with the risk shifts that enter the prequential decomposition. In this sense, monitoring is a channel-design problem rather than a search for a universal drift score. Section~\ref{subsec:misspecified-observable-feedback} tests this point in a misspecified feedback setting. There, even when the variables entering the feedback perturbation are known, the most informative observable channel for \(v_t\) and \(V_T\) is not obtained simply by retaining all of them. This illustrates that channel design is part of the observability problem rather than a mechanical consequence of access to more coordinates.

\section{Experimental Validation}
\label{sec:experiments}

This section validates the empirical quantities that organize our theory: the prequential gap, its sampling--drift decomposition, the additive drift--feedback budget, and observable Fisher motion under fixed monitoring channels. We begin with a linear--Gaussian environment where all terms are available in closed form. We then move to a nonlinear teacher--learner system with closed-loop feedback and ask whether the same additive structure continues to organize $\Delta_T^{\mathrm{rep}}$ beyond the analytically tractable regime. Next, we isolate Fisher contraction under known Gaussian monitoring channels. Finally, we turn to a misspecified feedback experiment on an evolving empirical population, where the full law is not available, and examine whether crude observable channels retain a risk-relevant Fisher signature of the one-step shifts $v_t$ and $V_T$.

\subsection{Linear--Gaussian demonstration: budgeted drift component and horizon scaling}
\label{sec:gaussian_validation}

We include a closed--form Gaussian model to instantiate the primitives appearing in the decomposition.  We consider Gaussian mean estimation with fixed covariance $\Sigma$ and a drifting environment state $\theta_t\in\mathbb{R}^d$.  At time $t$ we observe $x_t\sim\mathcal{N}(\theta_t,\Sigma)$ and deploy the predictor $f_t=\hat\mu_{t-1}$, where $\hat\mu_t$ is the online sample mean.

Our target quantity is the one--step--ahead prequential gap $\Delta_T^{\mathrm{rep}}=|\widehat R_T-R_T^+|$, together with the decomposition $\Delta_T^{\mathrm{rep}}\le \Delta_T^{\mathrm{sam}}+V_T$ from Section~\ref{sec:setup}. In this setting the population risk is $R(\theta,f)=\operatorname{tr}(\Sigma)+\|\theta-f\|_2^2$, so the above terms can be computed exactly. In this model, nothing is hidden.

\paragraph{Budgeted scaling of the drift component.}
To isolate the drift contribution, we examine how $V_T$ scales with the Fisher--consistent motion budget.  The environment evolves via an exogenous Fisher--budgeted increment plus an endogenous feedback component, yielding four regimes (iid, exogenous only, endogenous only, mixed).  For each run we record
\[
\bar d_T:=\frac{1}{T}\sum_{t=1}^T \|v^{\mathrm{exo}}_t\|_{\Sigma^{-1}},
\qquad
\bar\kappa_T:=\frac{1}{T}\sum_{t=1}^T \|\gamma u_t\|_{\Sigma^{-1}},
\qquad
\frac{C_T}{T}:=\bar d_T+\alpha\,\bar\kappa_T,
\]
and fit $\alpha^\star$ by linear regression. This directly tests the predicted additive organization of the drift term. Figure~\ref{fig:gaussian_two_panel}(a) shows a tight collapse of $V_T$ as an approximately linear function of $C_T/T$ across regimes (fit: $\alpha^\star\approx 2.89$, $R^2\approx 0.97$), confirming the predicted additivity of exogenous and feedback contributions in this closed--form regime.

\paragraph{Horizon scaling of sampling vs.\ drift.}
We next illustrate how the decomposition separates horizon effects.  Fixing a persistent bounded drift level (here $\delta_F=0.1$), we plot the empirical sampling term $\Delta_T^{\mathrm{sam}}$ together with $V_T$ as functions of $T$. Figure~\ref{fig:gaussian_two_panel}(b) shows that $\Delta_T^{\mathrm{sam}}$ decreases with $T$, whereas $V_T$ remains approximately stable, so the dominant contribution transitions from sampling- driven to drift--driven as the horizon grows.

\begin{figure}[t]
  \centering
  \begin{subfigure}[t]{0.49\linewidth}
    \centering
    \includegraphics[width=\linewidth]{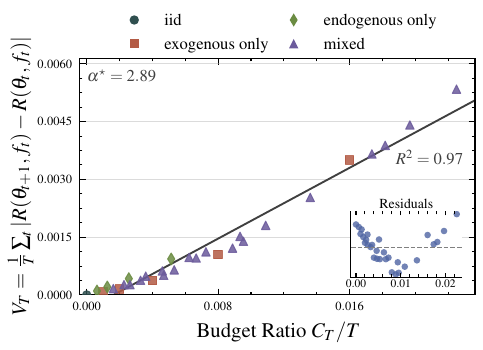}
    \caption{Drift component $V_T$ vs.\ budget ratio $C_T/T$
    (fit gives $\alpha^\star\approx 2.89$, $R^2\approx 0.97$).}
    \label{fig:gaussian_additivity}
  \end{subfigure}\hfill
  \begin{subfigure}[t]{0.49\linewidth}
    \centering
    \includegraphics[width=\linewidth]{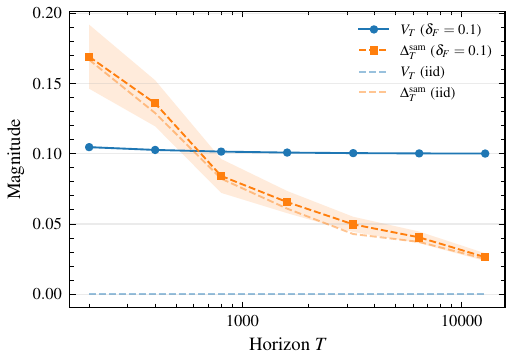}
    \caption{Component scaling vs.\ horizon $T$:
    $\Delta_T^{\mathrm{sam}}$ decays while $V_T$ is approximately constant
    (illustrated at $\delta_F=0.1$ with iid reference).}
    \label{fig:gaussian_components}
  \end{subfigure}
  \caption{Linear--Gaussian demonstrations for the decomposition
  $\Delta_T^{\mathrm{rep}}\le \Delta_T^{\mathrm{sam}}+V_T$.
  (a) The drift component $V_T$ scales linearly with the Fisher--consistent
  budget ratio $C_T/T$ across iid, exogenous, endogenous, and mixed regimes.
  (b) As $T$ increases, the sampling term decreases while the drift component
  remains stable, illustrating the separation of sampling and drift effects.}
  \label{fig:gaussian_two_panel}
\end{figure}

\subsection{Nonlinear neural-network validation: teacher--learner drift}
\label{sec:nn_stress}

We test the same structure in a nonlinear teacher--learner environment in which the learner is trained online and the environment state evolves through closed-loop feedback. At time $t$, covariates are sampled as
$x_t\sim \mathcal N(0,I_{d})$ and the teacher generates labels
$y_t=\langle \phi(x_t),\theta_t\rangle+\varepsilon_t$ with
$\varepsilon_t\sim\mathcal N(0,\sigma^2)$, where $\phi$ is a fixed nonlinear feature map and $\theta_t\in\mathbb R^{m}$ is the environment state. The learner is a two-layer MLP $f_t$ updated by streaming SGD on $(x_t,y_t)$. The environment update combines an exogenous Fisher--budgeted increment with an endogenous component that moves $\theta_t$ along a clipped disagreement-gradient direction between student predictions and the teacher rule. We report results over $12$ seeds with horizons
$T\in\{800,1600,3200,6400\}$, exogenous ratios
$C_{\mathrm{exo}}/T\in\{2.5,5,10,20,40,80\}\times 10^{-3}$, and feedback strengths
$\gamma\in\{0,0.0025,0.005,0.01,0.02\}$, using input dimension $d=5$, feature dimension $m=64$ (i.e., $\phi:\mathbb R^d\to\mathbb R^m$), hidden width $128$, and noise level $\sigma=0.1$.

\paragraph{Evaluation.}
Population MSE $R(\theta,f)$ is approximated by Monte Carlo evaluation on fresh draws from the teacher distribution. When evaluation is subsampled along a nonstationary trajectory, $\widehat R_T$ and $R_T^+$ must be computed on a common evaluation schedule; we enforce this by recording both quantities at identical evaluation times, including $t=T$.

\paragraph{Dependence on exogenous drift and feedback strength.}
Figure~\ref{fig:nn_panels} (left) plots $\Delta_T^{\mathrm{rep}}$ versus $C_{\mathrm{exo}}/T$ for each horizon, with separate curves for $\gamma$. Across horizons, $\Delta_T^{\mathrm{rep}}$ increases with $C_{\mathrm{exo}}/T$, and larger $\gamma$ generally yields a larger gap at fixed $C_{\mathrm{exo}}/T$. Thus the feedback channel acts as an additional source of distributional motion that can amplify the prequential gap in this nonlinear closed-loop setting.

\paragraph{Budget collapse with held-out horizon.}
We form the combined proxy
\[
\frac{C_T}{T}
\;:=\;
\frac{\sum_{t=1}^T d_t + \alpha \sum_{t=1}^T \kappa_t}{T},
\]
and estimate $\alpha$ by fitting the additivity plane
\[
\Delta_T^{\mathrm{rep}}
~\approx~
b_0 + b_s T^{-1/2}
+ b_1\Bigl(\tfrac{\sum_t d_t}{T}\Bigr)
+ b_2\Bigl(\tfrac{\sum_t \kappa_t}{T}\Bigr)
\]
on all horizons except a held-out $T_{\mathrm{hold}}=6400$, setting $\alpha=b_2/b_1$. In this experiment the fit yields $\alpha\simeq 2.67$ and achieves $R^2\simeq 0.81$ for the plane model on $\Delta_T^{\mathrm{rep}}$. Figure~\ref{fig:nn_panels} (right) plots $\Delta_T^{\mathrm{rep}}$ versus $C_T/T$ within each horizon using this single value of $\alpha$, including the held-out horizon. The within-horizon relationships are well approximated by a linear dependence on $C_T/T$, supporting the predicted additive organization of $\Delta_T^{\mathrm{rep}}$ by the combined Fisher-motion budget. The corresponding observed-versus-predicted plane diagnostic is reported in Appendix~\ref{app:exp-results}.

\begin{figure}[t]
  \centering
  \includegraphics[width=0.48\linewidth]{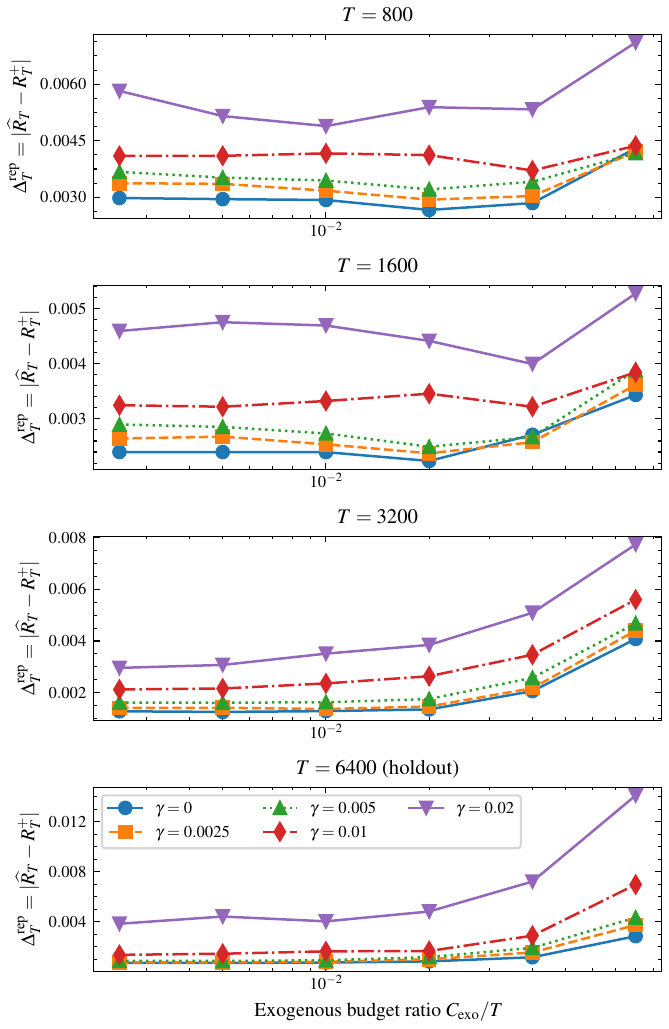}
  \hfill
  \includegraphics[width=0.48\linewidth]{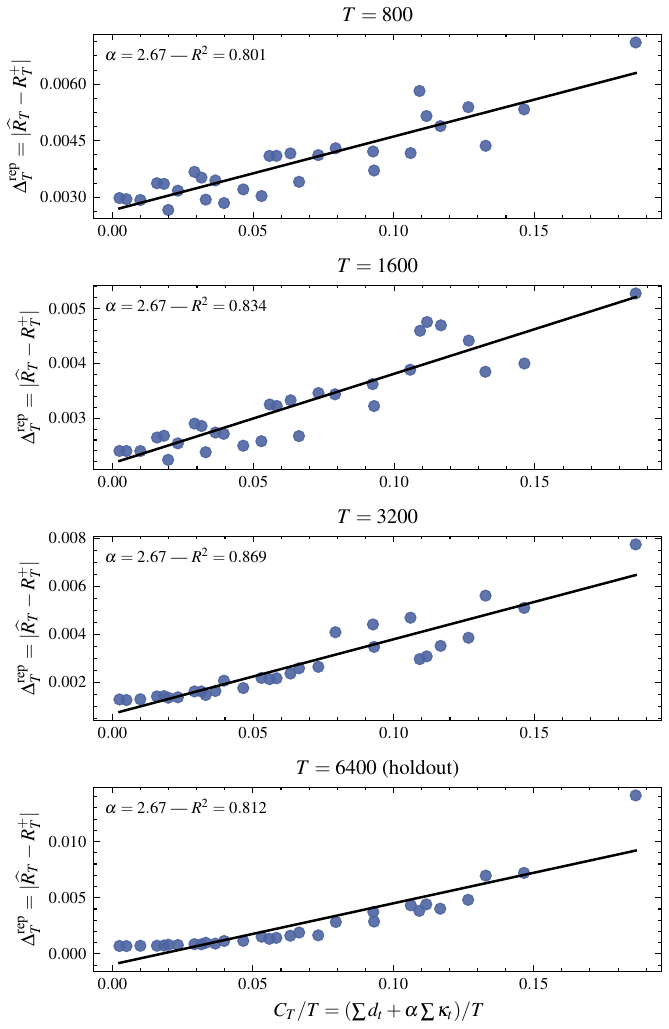}
  \caption{Neural-network teacher--learner validation.
  \emph{Left:} $\Delta_T^{\mathrm{rep}}=|\widehat R_T-R_T^+|$ versus the exogenous
  ratio $C_{\mathrm{exo}}/T$ for horizons $T\in\{800,1600,3200,6400\}$ and
  feedback strengths $\gamma$.
  \emph{Right:} Within-horizon collapse versus
  $C_T/T=(\sum_t d_t+\alpha\sum_t\kappa_t)/T$, where $\alpha$ is fitted on
  $T\neq 6400$ and evaluated on the held-out horizon.}
  \label{fig:nn_panels}
\end{figure}

\subsection{Closed-form check: observable Fisher contraction}
\label{subsec:closed-form-observable-fisher}

After testing the intrinsic drift--feedback decomposition, we turn to the observability question of Section~\ref{sec:observable_footprints}. We first isolate that question in a setting where both intrinsic and observed Fisher motion are available in closed form. Let \(P_t=N(\theta_t,\Sigma)\) evolve under the mixed exogenous/endogenous Gaussian drift--feedback dynamics above. Then
\[
    d_F(P_{t+1},P_t)=\|\theta_{t+1}-\theta_t\|_{\Sigma^{-1}} .
\]
We observe the system through a fixed linear Gaussian channel
\[
    o = Bx+\xi,\qquad \xi\sim N(0,\sigma_K^2 I_k).
\]
The induced law remains Gaussian,
\[
    Q_t=K_\sharp P_t
    =
    N(B\theta_t,\;B\Sigma B^\top+\sigma_K^2 I_k),
\]
so the channel Fisher step is
\[
    d_F(Q_{t+1},Q_t)
    =
    \|B(\theta_{t+1}-\theta_t)\|_{(B\Sigma B^\top+\sigma_K^2 I_k)^{-1}} .
\]

\begin{figure}[t]
    \centering

    \begin{subfigure}[t]{0.49\textwidth}
        \centering
        \includegraphics[width=\textwidth]{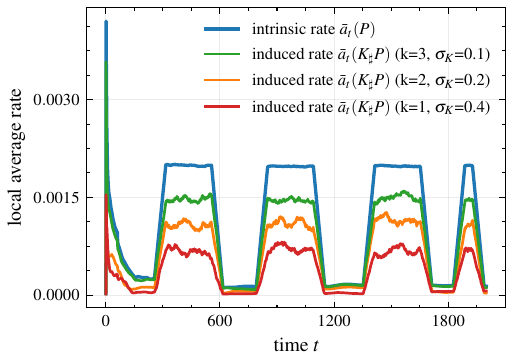}
        \caption{}
        \label{fig:observable-fisher-gaussian-rate}
    \end{subfigure}
    \hfill
    \begin{subfigure}[t]{0.49\textwidth}
        \centering
        \includegraphics[width=\textwidth]{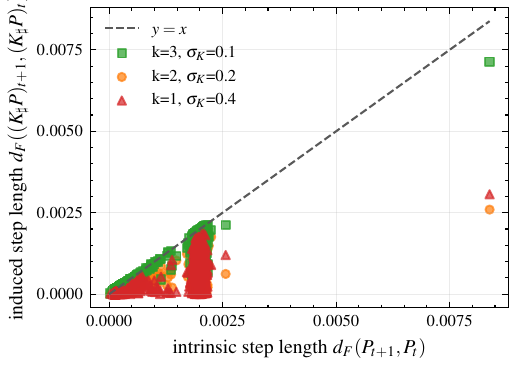}
        \caption{}
        \label{fig:observable-fisher-gaussian-contraction}
    \end{subfigure}

    \caption{
    Closed-form observable Fisher contraction.
    (a) Intrinsic local Fisher rate and induced rates under three fixed linear Gaussian monitoring channels. The exogenous drift budget is allocated in bursts, changing the local speed while keeping the total budget fixed.
    (b) Stepwise contraction under the same channels: each induced step length satisfies \(d_F(Q_{t+1},Q_t)\le d_F(P_{t+1},P_t)\). The observed rate therefore tracks temporal changes in drift speed, but only as a contracted, channel-dependent projection of the intrinsic motion.
    }
    \label{fig:observable-fisher-gaussian}
\end{figure}

Figure~\ref{fig:observable-fisher-gaussian} illustrates the two features most relevant to the observable-rate construction. First, the channel step lengths lie below the intrinsic step lengths, visualizing Fisher--Rao contraction under a fixed kernel. Second, when the same total exogenous budget is allocated in bursts, the trailing-window channel rates track the high- and low-speed intervals while remaining uniformly smaller. Thus the observed rate is not an estimate of the intrinsic rate; it is the Fisher rate of the pushed-forward process, with the amount of visible motion determined by the channel.

The next experiment asks the related question: when the intrinsic law is not available and the channel is deliberately crude, does the induced Fisher rate still preserve motion that is relevant to the deployed predictor's one-step risk?

\subsection{Misspecified real-data feedback: observable Fisher motion retains risk-relevant signal}
\label{subsec:misspecified-observable-feedback}

The preceding monitoring-channel experiment is closed form. Both the intrinsic Fisher--Rao motion and the channel-induced motion are available analytically. We next consider a less idealized setting in which the environment is represented by an evolving empirical population and Fisher motion is observed only through fixed low-dimensional summaries.

We use the public \emph{US Regional Sales Data} dataset from Kaggle~\citep{talhabu_us_regional_sales} and take unit price as the prediction target. Before feedback begins, a ridge predictor is fit on an initial split, so the closed-loop process does not start from an untrained model. At each round, the deployed predictor $f_t$ is evaluated on the current empirical population, and its predictions induce a nonlinear perturbation of the next population. The intervention is intentionally misspecified relative to the predictor class: it acts through a bounded nonlinear response of centered predictions and changes the joint empirical law over price, quantity, cost, and sales-channel subgroup. Seeded trajectory-level perturbations are included in addition to row-level noise, so repeated runs produce distinct feedback trajectories while preserving the same intervention structure.

The timing matches the prequential decomposition. We deploy $f_t$, move the environment, compare the same fixed $f_t$ across the transition, and only then update the learner. For each transition we compute
\[
    v_t =
    \left|
        R_{t+1}(f_t)-R_t(f_t)
    \right|,
    \qquad
    V_T = \frac{1}{T}\sum_t v_t ,
\]
together with the prequential gap $\Delta_T^{\mathrm{rep}}$. Here $R_t(f_t)$ denotes empirical-population risk of the fixed deployed predictor on the population at round $t$.

To compute observable Fisher motion, we apply fixed categorical monitoring channels. Each channel maps every row of the empirical population to a discrete state by binning selected observable coordinates, producing a histogram $\hat q_t$ on a finite probability simplex. For successive populations, the observed Fisher step is the Fisher--Rao distance between categorical laws,
\[
    d_F(\hat q_{t+1},\hat q_t)
    =
    2\arccos\!\left(
        \sum_j \sqrt{\hat q_{t+1,j}\hat q_{t,j}}
    \right),
\]
with a small pseudocount used only to avoid empty-bin degeneracies. Thus the experiment does not estimate Fisher motion of the full empirical law. It computes the Fisher motion of the empirical process after projection through a fixed monitoring channel.

We compare a null blind channel, a coarse score channel, and task-relevant channels. The null blind channel is a fixed row-partition control containing no prediction, target, or task information. The coarse score channel observes only bins of the deployed prediction score, representing output-only monitoring. The task-aligned channel bins the prediction score together with order quantity, unit cost, and sales-channel subgroup; ablated task channels omit one of these coordinates at a time. The goal is not to benchmark generic drift detection, but to test whether simple fixed channels preserve Fisher motion aligned with the transition-level risk shifts \(v_t\) that enter the prequential decomposition.

\begin{figure}[t]
    \centering
    \includegraphics[width=\textwidth]{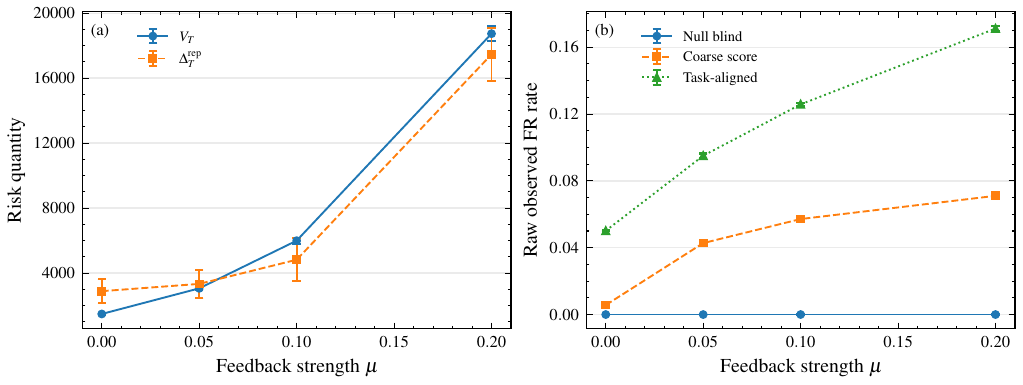}
    \caption{
    Misspecified feedback intervention. Left: increasing feedback strength $\mu$ increases the drift penalty $V_T$ and the prequential gap $\Delta_T^{\mathrm{rep}}$. Right: the same intervention increases raw observed Fisher rates for non-null monitoring channels, while the null blind channel remains flat. This is a manipulation check: the nonlinear feedback perturbation moves the empirical population in ways that affect one-step risk and are visible through suitable monitoring channels.
    }
    \label{fig:regional-feedback-manipulation}
\end{figure}

Figure~\ref{fig:regional-feedback-manipulation} verifies that the intervention moves the system. As $\mu$ increases, both $V_T$ and $\Delta_T^{\mathrm{rep}}$ increase, and raw observed Fisher rates increase for non-null channels. The null blind channel remains at zero because its fixed row partition does not change across rounds.

\begin{figure}[t]
    \centering
    \includegraphics[width=\textwidth]{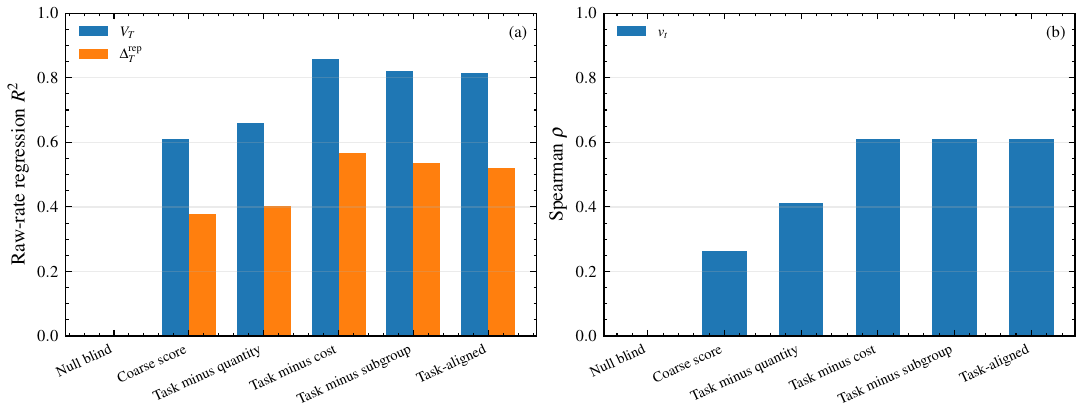}
    \caption{
    Channel dependence of observable Fisher motion under misspecified feedback. Left: run-level raw-rate regression $R^2$ for $V_T$ and $\Delta_T^{\mathrm{rep}}$. Right: transition-level Spearman association between per-step observed Fisher motion and $v_t$ over positive-feedback transitions. The null blind channel is a fixed row-partition negative control, while the coarse score channel is an output-level monitor. Task-relevant channels are substantially more informative, especially at the transition level. The reported rates use fixed categorical bins and are not calibrated estimators of intrinsic drift or run-level degradation.
    }
    \label{fig:regional-channel-association}
\end{figure}

Figure~\ref{fig:regional-channel-association} reports the main channel diagnostics. At the run level, raw observed Fisher rates are more strongly associated with the drift penalty $V_T$ than with the prequential gap $\Delta_T^{\mathrm{rep}}$, as expected from $\Delta_T^{\mathrm{rep}}\leq \Delta_T^{\mathrm{sam}}+V_T$. The strongest association with $V_T$ occurs for the task-minus-cost channel ($R^2\approx 0.85$), with task-minus-subgroup and task-aligned close behind ($R^2\approx 0.80$). Associations with $\Delta_T^{\mathrm{rep}}$ are weaker, reaching about $R^2\approx 0.52$--$0.56$ for the strongest task channels.

The transition-level analysis is the most direct test of risk-relevant observability. Over positive-feedback transitions, the coarse score channel has only moderate association with $v_t$ ($\rho\approx 0.27$), while task-relevant channels are substantially stronger: task-minus-cost, task-minus-subgroup, and task-aligned achieve $\rho\approx 0.62$--$0.64$. Thus monitoring only model outputs preserves some feedback-induced motion, but it is substantially less informative about one-step risk shifts than channels that retain task-relevant covariate structure. The observed Fisher signal is therefore not merely a consequence of increasing feedback strength. Channel choice determines how much local risk-relevant motion is preserved.

The ordering of the task channels is informative but requires interpretation. Retaining every feedback-related coordinate is not uniformly best: task-minus-cost gives the strongest run-level association with $V_T$, while the full task-aligned channel is strongest at the transition level. This reflects that the variables used to generate the feedback perturbation need not coincide with the directions through which the deployed predictor's risk is most sensitive.

Finally, it is important to note that the raw categorical rate is not a calibrated estimator of $C_T/T$, $V_T$, or run-level degradation. The bins are simple, fixed in advance, and not optimized for predictive performance. The conclusion is narrower: even under nonlinear misspecified feedback, unoptimized monitoring channels can retain a contracted but risk-relevant observable Fisher signal.

\section{Discussion and Outlook}
\label{sec:discussion}

The results in this paper recast learning under drift as a question of information speed. In stationary learning, additional samples reduce uncertainty about a fixed population quantity. In closed-loop learning, the population quantity itself moves. The relevant limitation is therefore not sample size alone, but whether the data-generating process moves slowly enough, in Fisher--Rao
geometry and in risk-relevant directions, for realized performance to remain
predictive of one-step-ahead performance. This is the sense in which prequential reproducibility is a finite resource: when the average motion rate $C_T/T$ is nonnegligible, part of the error is imposed by the trajectory rather than by estimation.

The rate perspective also changes how adaptation should be interpreted. A learner that reacts quickly to new data is not automatically more reproducible; its policy may also increase the motion of the environment it is trying to track. Conversely, a learner that appears stable in realized loss may still be operating in a high-drift regime if the observed losses are insensitive to the directions in which the data-generating law is moving. The decomposition into exogenous and policy-sensitive motion is useful precisely because it separates a moving target from a target made to move by the learner's own actions.

The main practical limitation is observability. The intrinsic budget $C_T$ is defined at the level of the data-generating law, while deployed systems usually observe only summaries, logs, predictions, or monitoring features. Section~\ref{sec:observable_footprints} shows that any fixed channel induces its own Fisher--Rao motion, and that this observed motion contracts relative to intrinsic motion. Thus observed Fisher motion should not be interpreted as a calibrated estimate of $C_T/T$. It is the information-geometric rate of the process as seen through a particular representation. Therefore, the intrinsic theorem and the observable-channel construction play different roles: the former gives a sharp bound when the full trajectory lies in a regular Fisher--Rao regime, while the latter can still provide an operational signal when only a regular projected law is available.

Monitoring therefore becomes a channel-design problem rather than a search for a universal drift detector. The goal is to choose representations whose contracted Fisher motion remains informative about one-step risk variation. This is the operational counterpart of the indistinguishability result in Section~\ref{subsec:lower_bound}: realized prequential evidence alone need not identify the drift effect that matters for the one-step-ahead target. The misspecified feedback experiment illustrates this projected version of the same issue. Output-only monitoring retained some signal, while task-relevant channels were more informative; however, richer channels were not uniformly better, since additional coordinates can introduce sparsity, nuisance variation, or cancellation. Thus channel design and calibration are separate operational layers, to be developed through replay, simulation, held-out reference regimes, or deployment-specific calibration curves relating observed channel motion to \(V_T\) and transition-level risk shifts.

The theory is intentionally local, but the locality is distributional rather than pathwise in the realized states. Individual observations, rewards, or states may change abruptly; what the bounds require is that the induced laws along the analyzed segment remain in a regular region where Fisher--Rao motion controls population-risk change and where the policy-sensitive response admits a controlled local comparison. Regime switches, support changes, or other nonlocal changes fall outside this local reproducibility regime unless treated piecewise or through an appropriate monitoring channel.

Several extensions are immediate. One is to characterize families of monitoring channels that are sufficient for risk-relevant drift, rather than merely contracted. Another is to develop adaptive or learned channels while preserving a fixed enough reference frame for Fisher-rate comparisons to remain meaningful. A third is to connect the policy-sensitive term to algorithm design. If a learner can choose among actions with different induced Fisher motion, then reproducibility becomes not only an evaluation property but a control objective. For adaptive systems, evaluation should therefore include not only current-stream loss, but also the rate at which the learner--environment loop spends Fisher--Rao motion relevant to future one-step performance.

\section*{Declarations}

\noindent\textbf{Funding.} No external funding supported this work.

\medskip
\noindent\textbf{Conflict of interest.} The author declares no competing interests.

\medskip
\noindent\textbf{Code availability.} Source code and figure-generation scripts for all experiments are available at
\url{https://github.com/fiazaich/Learning-under-Drift}.

\medskip
\noindent\textbf{Data availability.}
The \emph{US Regional Sales Data} dataset used in Section~\ref{subsec:misspecified-observable-feedback} is publicly available on Kaggle~\citep{talhabu_us_regional_sales}. The processed data, preprocessing scripts, experiment code, and figure-generation scripts are available in the accompanying repository.

\appendix

\section{Proofs}
\label{app:prelim_proofs}

For completeness we include proofs of the primary lemmas, theorems and equations. Throughout, constants may depend on the regularity and bounded-geometry constants from Assumption~\ref{assump:regularity}, but not on $T$ or on particular parameter values within compact subsets of $\Theta$.

\subsection{Proof of Lemma~\ref{lem:local_equivalence}
(Local geodesic equivalence)}\label{app:proof_local_equivalence}

\begin{proof}
Fix $\theta\in K$ and let $v\in T_\theta\Theta$ satisfy $\|v\|_{g_\theta}\le r_{\mathrm{inj}}$.
Let
\[
\gamma:[0,1]\to\Theta, \qquad \gamma(s):=\operatorname{Exp}_\theta(sv).
\]
Then $\gamma$ is a geodesic with $\gamma(0)=\theta$ and $\gamma(1)=\operatorname{Exp}_\theta(v)$ by definition of the exponential map.

\paragraph{Minimizing property inside the injectivity radius.}
By the definition of the injectivity radius at $\theta$, the restriction of the exponential map $\operatorname{Exp}_\theta$ to the open ball $B_{r_{\mathrm{inj}}}(0)\subset T_\theta\Theta$ is a diffeomorphism onto its image, and every point in $\operatorname{Exp}_\theta(B_{r_{\mathrm{inj}}}(0))$ is joined to $\theta$ by a \emph{unique} minimizing geodesic (see, e.g., standard consequences of the injectivity-radius definition; cf.\ \citet[Prop.~6.10]{Lee97}). Since $\|v\|_{g_\theta}\le r_{\mathrm{inj}}$, the endpoint $\operatorname{Exp}_\theta(v)$ lies in this normal neighborhood, hence $\gamma$ is the (unique) minimizing geodesic from $\theta$ to $\operatorname{Exp}_\theta(v)$ on $[0,1]$. In particular,
\[
d_F\bigl(\theta,\operatorname{Exp}_\theta(v)\bigr)=L(\gamma),
\]
where $L(\gamma)$ denotes the Riemannian length of $\gamma$.

\paragraph{Constant speed and length computation.}
Because $\gamma$ is a geodesic, its speed is constant:
$s\mapsto \|\dot\gamma(s)\|_{g_{\gamma(s)}}$ is constant on $[0,1]$.
Moreover, $\dot\gamma(0)=v$ (differentiate $\operatorname{Exp}_\theta(sv)$ at $s=0$), so the constant speed equals $\|v\|_{g_\theta}$:
\[
\|\dot\gamma(s)\|_{g_{\gamma(s)}}=\|\dot\gamma(0)\|_{g_\theta}=\|v\|_{g_\theta}
\qquad\text{for all } s\in[0,1].
\]
Therefore the length is
\[
L(\gamma)=\int_0^1 \|\dot\gamma(s)\|_{g_{\gamma(s)}}\,ds
= \int_0^1 \|v\|_{g_\theta}\,ds
= \|v\|_{g_\theta}.
\]

Combining the minimizing property with the length computation yields
\[
d_F\bigl(\theta,\operatorname{Exp}_\theta(v)\bigr)=L(\gamma)=\|v\|_{g_\theta},
\]
which is exactly~\eqref{eq:local_equiv}.
\end{proof}

\subsection{Proofs of Lemma~\ref{lem:subgaussian_martingale}
and Corollary~\ref{cor:subgaussian_expectation}
(Sub-Gaussian martingale deviation and expected deviation)}
\label{app:proof_lem_subg_martingale_cor_subg_exp}

The following is the standard Chernoff/mgf argument for conditionally
sub-Gaussian martingale differences; see, e.g.,
\citet[Section~2.2.2]{Wainwright2019}. We include the short derivation of the
expected-deviation bound because the constant is used in Theorem~\ref{thm:main_rep}.

\begin{proof}[Proof of Lemma~\ref{lem:subgaussian_martingale}]
Let \(S_T=\sum_{t=1}^T Z_t\) and \(\mathcal V_T=\sum_{t=1}^T\sigma_t^2\).
Iterating the conditional mgf bound gives, for every \(\lambda\in\mathbb R\),
\[
\mathbb E e^{\lambda S_T}
\le
\exp\!\left(\frac{\lambda^2\mathcal V_T}{2}\right).
\]
Hence, for \(\eta>0\) and \(\lambda>0\),
\[
\Pr(S_T\ge \eta)
\le
\exp\!\left(-\lambda\eta+\frac{\lambda^2\mathcal V_T}{2}\right).
\]
Optimizing at \(\lambda=\eta/\mathcal V_T\) yields
\[
\Pr(S_T\ge \eta)
\le
\exp\!\left(-\frac{\eta^2}{2\mathcal V_T}\right).
\]
Applying the same bound to \(-S_T\) gives
\[
\Pr(|S_T|\ge \eta)
\le
2\exp\!\left(-\frac{\eta^2}{2\mathcal V_T}\right).
\]
\end{proof}

\begin{proof}[Proof of Corollary~\ref{cor:subgaussian_expectation}]
Using the tail-integral identity and Lemma~\ref{lem:subgaussian_martingale},
\[
\mathbb E|S_T|
=
\int_0^\infty \Pr(|S_T|\ge \eta)\,d\eta
\le
\int_0^\infty
2\exp\!\left(-\frac{\eta^2}{2\mathcal V_T}\right)d\eta
=
\sqrt{2\pi\mathcal V_T}.
\]
Therefore
\[
\mathbb E\left|\frac1T\sum_{t=1}^T Z_t\right|
\le
\frac{\sqrt{2\pi\mathcal V_T}}{T}.
\]
If \(\sigma_t^2\le\sigma^2\) for all \(t\), then
\(\mathcal V_T\le\sigma^2T\), and hence
\[
\mathbb E\left|\frac1T\sum_{t=1}^T Z_t\right|
\le
\frac{\sqrt{2\pi}\sigma}{\sqrt T}.
\]
\end{proof}
\subsection{Proof of Equation~\ref{eq:budget_relation} (Information path bound)}\label{app:proof_eq_budget_relation}

\begin{proof}
Recall the discrete Fisher--Rao path length
\[
\mathcal A_T
~:=~
\sum_{t=1}^T d_F(\theta_{t+1},\theta_t).
\]
Work pathwise.

\paragraph{Triangle split into exogenous and controlled components.}
For each $t$, insert the baseline next state $F(\theta_t,0,\eta_t)$ and apply the triangle inequality:
\[
d_F(\theta_{t+1},\theta_t)
\le
d_F\!\big(\theta_{t+1},F(\theta_t,0,\eta_t)\big)
+
d_F\!\big(F(\theta_t,0,\eta_t),\theta_t\big).
\]
By definition of the exogenous drift magnitude (Section~\ref{sec:setup}),
\[
d_t
:= d_F\!\big(F(\theta_t,0,\eta_t),\theta_t\big),
\]
so it remains to control the first term.

\paragraph{Control the policy-sensitive displacement.}
By the local comparison assumption on $K$ (Equation~\eqref{eq:local-control-distance}), there exists $\alpha>0$ such that
\[
d_F\!\big(\theta_{t+1},F(\theta_t,0,\eta_t)\big)
\le
\alpha\,\kappa_t^{(\mathcal M)} + \epsilon_t,
\]
where $\kappa_t^{(\mathcal M)}$ is defined in \eqref{eq:kappa-grad} and $\epsilon_t:=\|r_t\|_{g_{F(\theta_t,0,\eta_t)}}$ satisfies the quadratic bound $\epsilon_t\le c\|u_t\|^2$ by \eqref{eq:eps_bound}.

Combining the previous two results gives the one-step motion bound
\[
d_F(\theta_{t+1},\theta_t)
\le
d_t + \alpha\,\kappa_t^{(\mathcal M)} + \epsilon_t.
\]

\paragraph{Sum over time.}
Summing over $t=1,\dots,T$ yields
\[
\mathcal A_T
=
\sum_{t=1}^T d_F(\theta_{t+1},\theta_t)
\le
\sum_{t=1}^T\bigl(d_t+\alpha\,\kappa_t^{(\mathcal M)}+\epsilon_t\bigr)
=
C_T + \sum_{t=1}^T \epsilon_t,
\]
where $C_T:=\sum_{t=1}^T (d_t+\alpha\,\kappa_t^{(\mathcal M)})$ is the intrinsic drift budget. The remainder bound $\epsilon_t\le c\|u_t\|^2$ follows from
\eqref{eq:eps_bound}.
\end{proof}

\subsection{Proof of Lemma~\ref{lem:VT_budget} (Summed drift penalty)}
\label{app:proof_lem_VT_budget}

\begin{proof}
Recall
\[
V_T(f,\pi)
:=\frac1T\sum_{t=1}^{T}\bigl|R(\theta_{t+1},f_t)-R(\theta_t,f_t)\bigr|.
\]
Fix any learner $(f,\pi)$ generating the trajectory $\{(\theta_t,u_t)\}_{t=1}^T$.

\paragraph{One-step risk motion bound.}
For each $t$, Assumption~\ref{assump:risk_lipschitz} applied with $\theta=\theta_t$ and $\theta'=\theta_{t+1}$ (and predictor $f_t$ fixed) yields the pathwise inequality
\begin{equation}
\label{eq:drift_feedback_step}
\bigl|R(\theta_{t+1},f_t)-R(\theta_t,f_t)\bigr|
\;\le\;
L_p\, d_F(\theta_{t+1},\theta_t).
\end{equation}
Summing \eqref{eq:drift_feedback_step} over $t=1,\dots,T$ and dividing by $T$ gives
\begin{equation}
\label{eq:drift_feedback_sum}
V_T(f,\pi)
\;\le\;
\frac{L_p}{T}\sum_{t=1}^{T} d_F(\theta_{t+1},\theta_t)
=
\frac{L_p}{T}\,\mathcal A_T,
\end{equation}
where $\mathcal A_T:=\sum_{t=1}^T d_F(\theta_{t+1},\theta_t)$ is the cumulative
Fisher--Rao path length.

\paragraph{Invoke the information path bound.}
By Equation~\ref{eq:budget_relation}, the path length satisfies the pathwise bound
\[
\mathcal A_T
\;\le\;
C_T+\sum_{t=1}^T \epsilon_t,
\qquad
\epsilon_t\le c\|u_t\|^2.
\]
Substituting into \eqref{eq:drift_feedback_sum} yields
\[
V_T(f,\pi)
\;\le\;
\frac{L_p}{T}\Big(C_T+\sum_{t=1}^T\epsilon_t\Big)
\;\le\;
\frac{L_p}{T}\,C_T+\frac{L_pc}{T}\sum_{t=1}^T\|u_t\|^2,
\]
which proves the first two results.

\paragraph{Take expectations.}
Taking expectations of the previous inequality and using linearity of expectation gives
\[
\mathbb E\,V_T(f,\pi)
\;\le\;
\frac{L_p}{T}\,\mathbb E[C_T]
\;+\;
\frac{L_pc}{T}\sum_{t=1}^T \mathbb E\|u_t\|^2,
\]
which is the final result.
\end{proof}

\subsection{Proof of Theorem~\ref{thm:main_rep} (Prequential reproducibility bound)}
\label{app:proof_thm_main_rep}

\begin{proof}
By the add--subtract decomposition \eqref{eq:rep_decomp_setup}, we have the pathwise bound
\[
\Delta_T^{\mathrm{rep}}(f,\pi)
\;\le\;
\Delta_T^{\mathrm{sam}}(f,\pi)
\;+\;
V_T(f,\pi).
\]
Taking expectations gives
\begin{equation}
\label{eq:thm_main_rep_exp_split}
\mathbb E\,\Delta_T^{\mathrm{rep}}(f,\pi)
\;\le\;
\mathbb E\,\Delta_T^{\mathrm{sam}}(f,\pi)
\;+\;
\mathbb E\,V_T(f,\pi).
\end{equation}

\paragraph{Sampling term.}
Recall $\Delta_T^{\mathrm{sam}}(f,\pi)=\left|\frac1T\sum_{t=1}^T Z_t\right|$ with
$Z_t:=\ell(f_t(x_t),y_t)-R(\theta_t,f_t)$ (Section~\ref{sec:setup}).
Under the endogenous--drift model \eqref{eq:drift_model},
$(Z_t)_{t=1}^T$ is a martingale difference sequence with respect to the filtration
$\{\mathcal F_t\}$, and Assumption~\ref{assump:regularity}(4) gives the conditional
$\sigma$--sub-Gaussian mgf bound required by Lemma~\ref{lem:subgaussian_martingale}.
Therefore Corollary~\ref{cor:subgaussian_expectation} yields
\begin{equation}
\label{eq:thm_main_rep_sampling_bound_app}
\mathbb E\,\Delta_T^{\mathrm{sam}}(f,\pi)
\;=\;
\mathbb E\left|\frac1T\sum_{t=1}^T Z_t\right|
\;\le\;
\frac{\sqrt{2\pi}\,\sigma}{\sqrt{T}}.
\end{equation}

\paragraph{Drift term.}
By Lemma~\ref{lem:VT_budget} (proved in Appendix~\ref{app:proof_lem_VT_budget}),
\[
V_T(f,\pi)
\;\le\;
\frac{L_p}{T}\,C_T
\;+\;
\frac{L_p c}{T}\sum_{t=1}^T \|u_t\|^2.
\]
Taking expectations gives
\begin{equation}
\label{eq:thm_main_rep_drift_bound_app}
\mathbb E\,V_T(f,\pi)
\;\le\;
\frac{L_p}{T}\,\mathbb E[C_T]
\;+\;
\frac{L_p c}{T}\sum_{t=1}^T \mathbb E\|u_t\|^2.
\end{equation}

\paragraph{Combine.}
Substituting \eqref{eq:thm_main_rep_sampling_bound_app} and
\eqref{eq:thm_main_rep_drift_bound_app} into \eqref{eq:thm_main_rep_exp_split}
yields the stated bound.
\end{proof}

\subsection{Proof of Theorem~\ref{thm:lower_bound}
(Sharpness of the prequential reproducibility rate)}
\label{app:proof_sharpness}

\begin{proof}
We give an explicit one-dimensional construction. Let
\(\mathcal Z=\{-1,+1\}\), and consider the Rademacher-coded Bernoulli family
\[
P_\theta(Z=1)=\frac{1+\sin\theta}{2},
\qquad
P_\theta(Z=-1)=\frac{1-\sin\theta}{2},
\qquad
\theta\in[-r_0,r_0],
\]
where \(r_0>0\) is chosen small. Since
\(\mathbb E_\theta Z=\sin\theta\), the Fisher information in the coordinate
\(\theta\) is
\[
I(\theta)
=
\frac{(\partial_\theta \sin\theta)^2}{1-\sin^2\theta}
=
\frac{\cos^2\theta}{\cos^2\theta}
=
1.
\]
Thus \(\theta\) is a Fisher-arclength coordinate and
\[
d_F(\theta,\theta')=|\theta-\theta'|
\]
on the interval.

Let the hypothesis class contain two deterministic predictors indexed by
\(s\in\{-1,+1\}\), with bounded loss
\[
\ell_s(z)=\frac12+\frac{s z}{4}.
\]
Then \(\ell_s(z)\in[1/4,3/4]\), and the population risk is
\[
R(\theta,s)
=
\mathbb E_\theta \ell_s(Z)
=
\frac12+\frac{s}{4}\sin\theta .
\]

Consider deterministic controlled dynamics
\[
\theta_{t+1}=\theta_t+u_t,
\qquad
\eta_t\equiv0.
\]
Then \(d_t\equiv0\), \(J_uF\equiv1\), the second-order remainder vanishes, and
\[
\kappa_t^{(M)}=|u_t|,
\qquad
C_T=\alpha\sum_{t=1}^T |u_t|.
\]

Choose controls so that the path alternates between \(0\) and \(\delta\):
\[
u_t=
\begin{cases}
+\delta, & t \text{ odd},\\
-\delta, & t \text{ even},
\end{cases}
\qquad
\delta:=\frac{C}{\alpha T}.
\]
Taking \(C/T\le \alpha r_0\) keeps the trajectory inside \([0,\delta]\subseteq[-r_0,r_0]\), and
\[
C_T=\alpha\sum_{t=1}^T |u_t|=C.
\]
Deploy the deterministic predictor sequence
\[
s_t=\operatorname{sign}(u_t).
\]
This sequence is history-independent and hence \(\mathcal F_{t-1}\)-measurable. For every step,
\[
R(\theta_{t+1},s_t)-R(\theta_t,s_t)
=
\frac{s_t}{4}\{\sin\theta_{t+1}-\sin\theta_t\}
=
\frac14\left|\sin\theta_{t+1}-\sin\theta_t\right|.
\]
For \(r_0\) sufficiently small, \(\cos\theta\ge 1/2\) on \([-r_0,r_0]\), so by the mean value theorem,
\[
\left|\sin\theta_{t+1}-\sin\theta_t\right|
\ge
\frac12 |\theta_{t+1}-\theta_t|
=
\frac12 |u_t|.
\]
Therefore
\[
R(\theta_{t+1},s_t)-R(\theta_t,s_t)
\ge
\frac18 |u_t|.
\]
Averaging over time gives
\[
R_T^+-R_T^{\mathrm{traj}}
=
\frac1T\sum_{t=1}^T
\{R(\theta_{t+1},s_t)-R(\theta_t,s_t)\}
\ge
\frac{1}{8T}\sum_{t=1}^T |u_t|
=
\frac{C}{8\alpha T}.
\]

Since \(\widehat R_T\) is the empirical average of losses observed under
\(P_{\theta_t}\), while \(R_T^{\mathrm{traj}}\) is the corresponding same-time
population average,
\[
\mathbb E_P[\widehat R_T-R_T^{\mathrm{traj}}]=0.
\]
Hence
\[
\mathbb E_P[\widehat R_T-R_T^+]
=
-\left(R_T^+-R_T^{\mathrm{traj}}\right).
\]
By Jensen's inequality,
\[
\mathbb E_P\Delta_T^{\mathrm{rep}}
=
\mathbb E_P|\widehat R_T-R_T^+|
\ge
\left|\mathbb E_P[\widehat R_T-R_T^+]\right|
=
R_T^+-R_T^{\mathrm{traj}}
\ge
\frac{C}{8\alpha T}.
\]
This proves the \(C/T\) lower bound.

For the sampling term, restrict to the stationary subclass \(u_t\equiv0\),
\(\theta_t\equiv0\), and \(s_t\equiv+1\). Then \(C_T=0\),
\(R_T^+=R_T^{\mathrm{traj}}\), and
\[
\Delta_T^{\mathrm{rep}}
=
\left|\widehat R_T-R_T^{\mathrm{traj}}\right|
=
\frac14\left|\frac1T\sum_{t=1}^T Z_t\right|,
\]
where \(Z_t\) are iid Rademacher variables. By the Khintchine inequality \citep{khintchine1923},
there exists a universal constant \(c_{\mathrm K}>0\) such that
\[
\mathbb E\left|\sum_{t=1}^T Z_t\right|
\ge
c_{\mathrm K}\sqrt T.
\]
Therefore
\[
\mathbb E_P\Delta_T^{\mathrm{rep}}
\ge
\frac{c_{\mathrm K}}{4\sqrt T}.
\]

Taking \(\mathcal P_C^{\mathrm{hard}}\) to be the union of the drift subclass
and the stationary subclass gives
\[
\sup_{P\in\mathcal P_C^{\mathrm{hard}}}
\mathbb E_P\Delta_T^{\mathrm{rep}}
\ge
c\max\left(T^{-1/2},\frac{C}{T}\right),
\]
for a constant \(c>0\) depending only on fixed model constants and \(\alpha\).
The losses are bounded, so the conditional sub-Gaussian sampling condition
holds. Moreover,
\[
|\partial_\theta R(\theta,s)|=\frac14|\cos\theta|\le \frac14,
\]
and since \(d_F(\theta,\theta')=|\theta-\theta'|\), Assumption~\ref{assump:risk_lipschitz}
holds with \(L_p=1/4\).

Both subclasses satisfy the local regularity assumptions, remain in the compact
Fisher chart, and have vanishing second-order remainder. This proves the theorem.
\end{proof}

\subsection{Proof of Theorem~\ref{thm:indistinguishability} (Information-theoretic Indistinguishability)}
\label{app:proof_invisibility}

\begin{proof}
We follow standard minimax lower-bound arguments based on two-point and multiple-hypothesis testing reductions; see \citet[Ch.~2]{Tsybakov2009}.

We construct an explicit subclass $\mathcal P_C^{\mathrm{hide}}$ of drift--feedback processes satisfying Assumptions~\ref{assump:regularity}--\ref{assump:risk_lipschitz} and $\mathbb E_P[C_T]\le C$ on which estimating
\[
R_T^{+}(P):=\frac1T\sum_{t=1}^T R(\theta_{t+1},f_t),
\qquad
\mathcal L_t(\theta):=R(\theta,f_t),
\]
so that $R_T^{+}(P)=\frac1T\sum_{t=1}^T \mathcal L_t(\theta_{t+1})$ incurs minimax error $\gtrsim \max(T^{-1/2},C/T)$.

\paragraph{Local model and Fisher geometry.}
Work in a one-dimensional regular exponential family $\{p_\theta\}$ on the observation space (equivalently, take $\mathcal Y$ trivial and write $(x_t,y_t)\equiv x_t$):
\[
p_\theta(x)=\exp(\theta x-A(\theta))\,h(x),\qquad \theta\in\Theta\subset\mathbb R,
\]
restricted to a compact interval on which the Fisher information $I(\theta)=A''(\theta)$ is bounded above and below. This satisfies Assumption~\ref{assump:regularity}.
Introduce the Fisher arclength coordinate
\[
s(\theta):=\int_{\theta_0}^{\theta}\sqrt{I(u)}\,du,
\]
and restrict to a neighborhood of $\theta_0$ where $s$ is a diffeomorphism. In the $s$-coordinate the Fisher metric is identically $1$, hence locally
\[
d_F(\theta,\theta')=|s(\theta)-s(\theta')|.
\]
For convenience we work in arclength coordinates and write this coordinate again as $\theta$, so that $d_F(\theta,\theta')=|\theta-\theta'|$ on the neighborhood under consideration.

\paragraph{Deterministic controlled dynamics and budget.}
Consider deterministic controlled dynamics
\[
\theta_{t+1}=F(\theta_t,u_t,\eta_t)=\theta_t+u_t,\qquad \eta_t\equiv 0.
\]
Then $d_t\equiv 0$, $J_uF\equiv 1$, and the Taylor remainder vanishes. In arclength coordinates ($g\equiv 1$) we have $\kappa_t^{(\mathcal M)}=|u_t|$, and therefore
\[
C_T=\alpha\sum_{t=1}^T |u_t|.
\]

\paragraph{Block construction and codebook.}
Fix a step size $\delta>0$ and let $m\le \lfloor T/2\rfloor$. Partition the horizon into $m$ disjoint two-step blocks $B_j=\{2j-1,2j\}$, $j=1,\dots,m$, and set $u_t\equiv 0$ for $t>2m$.

Let $V\subset\{-1,+1\}^m$ be a binary code with minimum Hamming distance at least $\beta m$ and cardinality $|V|\ge 2^{\zeta m}$ for constants $\beta,\zeta>0$ \citep[Lemma~2.9]{Tsybakov2009}.
For each $v=(v_1,\dots,v_m)\in V$ define deterministic controls
\[
u_{2j-1}^{(v)}=v_j\delta,\qquad u_{2j}^{(v)}=-v_j\delta,\qquad u_t^{(v)}=0 \ (t>2m).
\]
Starting from $\theta_1=\theta_0$, the resulting trajectory satisfies
\[
\theta_{2j-1}^{(v)}=\theta_0,\qquad \theta_{2j}^{(v)}=\theta_0+v_j\delta,\qquad
\theta_{2j+1}^{(v)}=\theta_0,
\]
and $\theta_t^{(v)}=\theta_0$ for $t>2m$. Each block performs a geodesic excursion out-and-back of length $2\delta$, so
\[
C_T(P_v)=\alpha\sum_{t=1}^T |u_t^{(v)}|=2\alpha m\delta.
\]
Choose
\begin{equation}
\label{eq:choose_m_delta}
m:=\left\lfloor \frac{C}{2\alpha\delta}\right\rfloor\wedge \left\lfloor\frac{T}{2}\right\rfloor
\qquad\text{so that}\qquad C_T(P_v)\le C\ \ \text{for all }v\in V.
\end{equation}
(In this subclass $C_T$ is deterministic, hence $\mathbb E_{P_v}[C_T]\le C$.)

\paragraph{Separation in the one-step-ahead risk functional.}
Fix a learner sequence $(f_t)_{t=1}^T$ and loss/model pair such that the population risk has a nonzero local slope at $\theta_0$ on the indices $t=2j-1$ used below: there exists $c_0>0$ such that $|\partial_\theta \mathcal L_{2j-1}(\theta_0)|\ge c_0$ for all $j\le m$. Then for sufficiently small $\delta$,
\[
\mathcal L_{2j-1}(\theta_0+\delta)-\mathcal L_{2j-1}(\theta_0-\delta)
=2\delta\,\partial_\theta \mathcal L_{2j-1}(\theta_0)+O(\delta^2),
\]
so there exists $c_\ell>0$ such that whenever $v_j\neq w_j$,
\[
\bigl|\mathcal L_{2j-1}(\theta_{2j}^{(v)})-\mathcal L_{2j-1}(\theta_{2j}^{(w)})\bigr|
=
\bigl|\mathcal L_{2j-1}(\theta_0+v_j\delta)-\mathcal L_{2j-1}(\theta_0+w_j\delta)\bigr|
\ge c_\ell\,\delta.
\]
Because $v$ and $w$ differ in at least $\beta m$ blocks, and because $R_T^+(P_v)=\frac1T\sum_{t=1}^T \mathcal L_t(\theta_{t+1}^{(v)})$ depends on the excursions exactly at the indices $t=2j-1$ via $\theta_{t+1}^{(v)}=\theta_{2j}^{(v)}$, we obtain the packing separation
\begin{equation}
\label{eq:sep_RTplus}
|R_T^{+}(P_v)-R_T^{+}(P_w)|
\;\ge\;
\frac{1}{T}\sum_{j:\,v_j\neq w_j}
\bigl|\mathcal L_{2j-1}(\theta_{2j}^{(v)})-\mathcal L_{2j-1}(\theta_{2j}^{(w)})\bigr|
\;\ge\;
\frac{\beta m c_\ell \delta}{T}.
\end{equation}
With the choice \eqref{eq:choose_m_delta}, this lower bound is of order $(m\delta)/T\asymp C/T$.

\paragraph{KL divergence control.}
For each $v$, the state path $\{\theta_t^{(v)}\}$ is deterministic. Conditional on $\theta_t^{(v)}$, the observation at time $t$ is drawn from $p_{\theta_t^{(v)}}$, so the joint law factorizes and
\[
D_{\mathrm{KL}}(P_v\Vert P_w)
=
\sum_{t=1}^T D_{\mathrm{KL}}\!\bigl(p_{\theta_t^{(v)}}\Vert p_{\theta_t^{(w)}}\bigr).
\]
The trajectories differ only at times $t=2j$ where $v_j\neq w_j$. On a compact neighborhood, regular exponential families satisfy a uniform quadratic KL bound: there exists $C_I<\infty$ such that for sufficiently small $\delta$,
\[
D_{\mathrm{KL}}(p_{\theta_0+\delta}\Vert p_{\theta_0-\delta})\le C_I\,\delta^2.
\]
Hence
\begin{equation}
\label{eq:kl_bound}
D_{\mathrm{KL}}(P_v\Vert P_w)\le C_I\,m\,\delta^2.
\end{equation}

\paragraph{Fano lower bound for the $C/T$ term.}
Choose $\delta$ small enough (depending only on fixed model constants) so that the Fano condition holds, e.g. $C_I\,m\,\delta^2 \le \tfrac{1}{8}\log|V|$, using \eqref{eq:kl_bound} and $\log|V|\ge \zeta m$.
Applying the multi-hypothesis Fano inequality to the packing
$\{P_v\}_{v\in V}\subset\mathcal P_C^{\mathrm{hide}}$ and the separation \eqref{eq:sep_RTplus} yields
\[
\inf_{\widehat R_T}\sup_{P\in\mathcal P_C^{\mathrm{hide}}}\mathbb E_P\bigl|\widehat R_T-R_T^{+}(P)\bigr|
\;\gtrsim\;
\frac{m\delta}{T}
\;\asymp\;
\frac{C}{T}.
\]

\paragraph{The $T^{-1/2}$ term from a stationary subclass.}
To obtain the classical term, restrict further to a stationary subclass ($C_T\equiv 0$) where $\theta_t\equiv\theta$ is fixed over time. Then $R_T^{+}(P)=\frac1T\sum_{t=1}^T \mathcal L_t(\theta)$ is a one-dimensional smooth functional of $\theta$. A standard two-point Le Cam argument over $\theta\in\{\theta_0\pm \delta/\sqrt{T}\}$ (using the same local quadratic KL control) gives
\[
\inf_{\widehat R_T}\sup_{P:\,C_T\equiv 0}\mathbb E_P\bigl|\widehat R_T-R_T^{+}(P)\bigr|
\;\gtrsim\; T^{-1/2}.
\]

\paragraph{Combine regimes.}
Since the stationary subclass is contained in $\mathcal P_C^{\mathrm{hide}}$ for any $C\ge 0$, combining the two lower bounds yields
\[
\inf_{\widehat R_T}
\sup_{P\in\mathcal P_C^{\mathrm{hide}}}
\mathbb E_P\bigl|\widehat R_T-R_T^{+}(P)\bigr|
\;\gtrsim\;
\max\!\left(T^{-1/2},\,\frac{C}{T}\right),
\]
which proves the theorem (absorbing constants into $c_1$).
\end{proof}

\subsection{Proofs of Limiting Regimes
(Stationary and Classical Limits)}
\label{app:reductions}

\begin{lemma}[Stationary (iid) regime]
\label{lem:iid_limit}
If $d_t\equiv 0$ and $J_uF\equiv 0$ for all $t$, then $\theta_t=\theta_0$ for all $t$, so $p_{\theta_t}=p_{\theta_0}$ and the samples are drawn from a stationary law.
Moreover $V_T(f,\pi)=0$ and
\[
\mathbb E\,\Delta_T^{\mathrm{rep}}(f,\pi)
\;\le\;
\mathbb E\,\Delta_T^{\mathrm{sam}}(f,\pi)
\;\le\;
\frac{\sqrt{2\pi\,\mathcal V_T}}{T}
\;\le\;
\frac{\sqrt{2\pi}\,\sigma}{\sqrt{T}},
\]
under the standing $\sigma$--sub-Gaussian increment condition.
\end{lemma}

\begin{proof}
If $d_t\equiv 0$ and $J_uF\equiv 0$, then the update does not move the environment, hence $\theta_{t+1}=\theta_t$ and $\theta_t=\theta_0$ for all $t$. Therefore $R(\theta_{t+1},f_t)=R(\theta_t,f_t)$ for every $t$ and $V_T(f,\pi)=0$.
The decomposition in the main text gives
$\Delta_T^{\mathrm{rep}}(f,\pi)\le \Delta_T^{\mathrm{sam}}(f,\pi)+V_T(f,\pi)
=\Delta_T^{\mathrm{sam}}(f,\pi)$.
The sampling term is an average of martingale differences, so the stated bound follows from Corollary~\ref{cor:subgaussian_expectation}.
\end{proof}

\begin{lemma}[Exogenous-drift regime]
\label{lem:exo_limit}
If $J_uF\equiv 0$ (no policy-induced motion), then the drift penalty is controlled purely by the exogenous budget and
\[
V_T(f,\pi)\;\le\;\frac{L_p}{T}\,C_T
\qquad\text{and}\qquad
\mathbb E\,\Delta_T^{\mathrm{rep}}(f,\pi)
\;\le\;
\mathbb E\,\Delta_T^{\mathrm{sam}}(f,\pi)
+\frac{L_p}{T}\,\mathbb E[C_T].
\]
In particular, when $C_T=\sum_{t=1}^T d_t$ in this regime, this matches the classical ``variation budget'' penalty scaling $\sum_t d_t/T$.
\end{lemma}

\begin{proof}
With $J_uF\equiv 0$, the endogenous component vanishes and the controls do not contribute to motion. Lemma~\ref{lem:VT_budget} gives
$V_T(f,\pi)\le \frac{L_p}{T}\big(C_T+\sum_t\epsilon_t\big)$. Combine with $\Delta_T^{\mathrm{rep}}\le \Delta_T^{\mathrm{sam}}+V_T$ and take expectations.
\end{proof}

\begin{lemma}[Adaptive-data / purely endogenous regime]
\label{lem:ada_limit}
If $d_t\equiv 0$ (no exogenous drift), then all motion is policy-induced and
\[
V_T(f,\pi)
\;\le\;
\frac{L_p}{T}\,C_T+\frac{L_p c}{T}\sum_{t=1}^T \|u_t\|^2,
\]
hence
\[
\mathbb E\,\Delta_T^{\mathrm{rep}}(f,\pi)
\;\le\;
\mathbb E\,\Delta_T^{\mathrm{sam}}(f,\pi)
+\frac{L_p}{T}\,\mathbb E[C_T]+\frac{L_p c}{T}\sum_{t=1}^T \mathbb E\|u_t\|^2.
\]
\end{lemma}

\begin{proof}
This is immediate from Lemma~\ref{lem:VT_budget} after setting $d_t\equiv 0$, and then combining with $\Delta_T^{\mathrm{rep}}\le \Delta_T^{\mathrm{sam}}+V_T$ followed by expectations.
\end{proof}

\paragraph{Synthesis.}
In all regimes,
\[
\mathbb E\,\Delta_T^{\mathrm{rep}}(f,\pi)
\;\le\;
\mathbb E\,\Delta_T^{\mathrm{sam}}(f,\pi)
+
\mathbb E\,V_T(f,\pi)
\;=\;
O\!\left(T^{-1/2}+\frac{\mathbb E[C_T]}{T}\right)
\quad\text{(up to the control-energy remainder).}
\]
Thus $T^{-1/2}$ recovers the classical sampling rate when the environment is stationary, while $\mathbb E[C_T]/T$ quantifies the additional drift contribution induced by risk-relevant distributional motion.

\section{Experimental Details}
\label{app:experiment-details}

This appendix summarizes the experimental protocols used for the figures in Section~\ref{sec:experiments}. All experiments were implemented in Python~3.11 using \texttt{NumPy}, \texttt{SciPy}, and \texttt{PyTorch} for the nonlinear learner. Per-run integer seeds were used for reproducibility.

\paragraph{Recorded quantities and timing.}
Across experiments we record the learner trajectory $\{f_t\}$, the environment trajectory $\{\theta_t\}$ when available, the empirical trajectory loss
\[
\hat R_T=\frac1T\sum_{t=1}^T \ell(f_t(x_t),y_t),
\]
and the corresponding same-time population risk
\[
R_T=\frac1T\sum_{t=1}^T R(\theta_t,f_t).
\]
When drift is present we also record the prequential target
\[
R_T^+=\frac1T\sum_{t=1}^T R(\theta_{t+1},f_t),
\qquad
\Delta_T^{\mathrm{rep}}=|\hat R_T-R_T^+|.
\]
The drift penalty is
\[
V_T=\frac1T\sum_{t=1}^T
\bigl|R(\theta_{t+1},f_t)-R(\theta_t,f_t)\bigr|.
\]
All predictors and controls are constructed to be $\mathcal F_{t-1}$-measurable. In particular, $f_t$ is fixed when comparing $R(\theta_t,f_t)$ and $R(\theta_{t+1},f_t)$, matching the decomposition
$\Delta_T^{\mathrm{rep}}\leq \Delta_T^{\mathrm{sam}}+V_T$.

\subsection{Linear--Gaussian experiment}
\label{app:exp-gaussian}

The linear--Gaussian experiment uses a Gaussian mean-estimation model with environment state $\theta_t\in\mathbb R^d$,
\[
x_t\sim \mathcal N(\theta_t,\Sigma),
\qquad
\ell(f,x)=\|x-f\|_2^2.
\]
The deployed predictor is the online mean $f_t=\hat\mu_{t-1}$. Population risk is available in closed form:
\[
R(\theta,\mu)=\mathrm{tr}(\Sigma)+\|\theta-\mu\|_2^2.
\]
Thus $\hat R_T$, $R_T$, $R_T^+$, $\Delta_T^{\mathrm{sam}}$, $\Delta_T^{\mathrm{rep}}$, and $V_T$ are computed directly.

For the budget-collapse experiment in Figure~3(a), we use $d=5$, $\Sigma=I_d$, $T=2000$, and 12 seeds. Exogenous increments are generated by drawing random directions and rescaling to prescribed Fisher length. The endogenous component is history-measurable, with $u_t=-k\hat\mu_{t-1}$, $k=0.25$, and feedback strength $\gamma$. We sweep
\[
C_{\mathrm{exo}}\in\{0,2,4,8,16,32\},
\qquad
\gamma\in\{0,0.01,0.02,0.04,0.08\}.
\]
We log
\[
\bar d=\frac1T\sum_t \|v_t^{\mathrm{exo}}\|_{\Sigma^{-1}},
\qquad
\bar\kappa=\frac1T\sum_t \|\gamma u_t\|_{\Sigma^{-1}},
\]
and fit $\alpha_\star$ so that $\bar d+\alpha_\star\bar\kappa$ organizes $V_T$.

For the horizon-scaling experiment in Figure~3(b), we use a reflecting-walk construction to keep $\theta_t$ in a compact region while maintaining approximately constant per-step Fisher motion $\delta_F$. We evaluate
\[
T\in\{200,400,800,1600,3200,6400,12800\},
\qquad
\delta_F\in\{0,0.1,0.2\},
\]
with 40 seeds.

\subsection{Nonlinear teacher--learner experiment}
\label{app:exp-nonlinear}

The nonlinear experiment uses a teacher--learner model with inputs
$x_t\sim\mathcal N(0,I_d)$, $d=5$, and labels
\[
y_t=\phi(x_t)^\top\theta_t+\sigma\xi_t,
\qquad
\xi_t\sim\mathcal N(0,1),
\qquad
\sigma=0.1.
\]
The feature map $\phi:\mathbb R^d\to\mathbb R^m$ has $m=64$ and is fixed throughout the experiment. The learner is a two-layer MLP with hidden width 128, trained online by single-sample SGD on squared loss.

Fisher motion in teacher space is approximated using a fixed Fisher--Gram matrix
\[
G \approx \frac{1}{N\sigma^2}\Phi(X_p)^\top\Phi(X_p),
\]
with a small ridge term for numerical stability. Exogenous drift is applied in random Fisher-normalized directions with total budget $C_{\mathrm{exo}}$. Policy-sensitive drift is applied along a normalized teacher--student disagreement direction, and the resulting Fisher length is logged as $\kappa_t^{(\mathcal M)}$.

Population MSE is estimated on fresh Monte Carlo batches:
\[
R(\theta,f)
=
\mathbb E[(f(x)-\phi(x)^\top\theta)^2]+\sigma^2.
\]
The same evaluation schedule is used for $\hat R_T$ and $R_T^+$, and paired risks $R(\theta_t,f_t)$ and $R(\theta_{t+1},f_t)$ are evaluated with the same covariate batch when computing $V_T$.

We sweep
\[
T\in\{800,1600,3200,6400\},
\qquad
C_{\mathrm{exo}}/T\in\{2.5,5,10,20,40,80\}\times 10^{-3},
\]
\[
\gamma\in\{0,0.0025,0.005,0.01,0.02\},
\]
using 12 seeds. The additivity plane is fit on all horizons except $T_{\mathrm{hold}}=6400$:
\[
\Delta_T^{\mathrm{rep}}
\approx
b_0+b_sT^{-1/2}
+b_1\bar d+b_2\bar\kappa,
\]
and the combined rate uses $\alpha_\star=b_2/b_1$.

\subsection{Closed-form observable Fisher contraction}
\label{app:exp-observable-fisher}

The closed-form observable-channel experiment uses the Gaussian location family
\[
P_t=\mathcal N(\theta_t,\Sigma),
\qquad
\Sigma=I_d,\quad d=5.
\]
Intrinsic Fisher--Rao step lengths are
\[
d_F(P_{t+1},P_t)=\|\theta_{t+1}-\theta_t\|_{\Sigma^{-1}}.
\]
The monitoring channel is a fixed linear Gaussian kernel
\[
o=Bx+\xi,
\qquad
\xi\sim\mathcal N(0,\sigma_K^2I_k),
\]
so the pushed-forward law is
\[
Q_t=K_\#P_t
=
\mathcal N(B\theta_t,\;B\Sigma B^\top+\sigma_K^2I_k),
\]
and the induced Fisher step is
\[
d_F(Q_{t+1},Q_t)
=
\|B(\theta_{t+1}-\theta_t)\|_{(B\Sigma B^\top+\sigma_K^2I_k)^{-1}}.
\]
Figure~5 reports both the local-rate comparison and the stepwise contraction check across three fixed choices of $(k,\sigma_K)$. To visualize rate changes, the same total exogenous budget is allocated in alternating high- and low-speed intervals, and local rates are shown using trailing-window averages.

\subsection{Misspecified feedback experiment}
\label{app:exp-misspecified-feedback}

We use the public \emph{US Regional Sales Data} dataset from Kaggle~\citep{talhabu_us_regional_sales}. Unit price is the prediction target, with order quantity, unit cost, and sales channel used in the feedback and monitoring construction. A ridge predictor is fit before feedback begins.

Each round uses the timing
\[
    f_t\ \text{deployed on}\ D_t,
    \qquad
    D_t\mapsto D_{t+1},
    \qquad
    R_t(f_t),R_{t+1}(f_t)\ \text{recorded},
    \qquad
    f_t\mapsto f_{t+1}.
\]
The same fixed predictor is therefore evaluated before and after the population transition.

The feedback perturbation is driven by centered deployed predictions. If
\[
    r_t(x)=\hat y_t(x)-|D_t|^{-1}\sum_{x'\in D_t}\hat y_t(x'),
\]
then the feedback response uses a bounded nonlinear score
\[
    h_t(x)=\tanh\!\left(r_t(x)/s_t\right),
\]
where $s_t$ is the empirical scale of the centered predictions. The perturbation changes unit price, order quantity, and unit cost, with sales-channel subgroup entering the heterogeneous response. Row-level noise and seeded round-level perturbations are included. Feedback strength is swept over
\[
    \mu\in\{0,0.05,0.10,0.20\},
\]
with multiple seeds per setting.

For each transition we compute
\[
    v_t=\bigl|R_{t+1}(f_t)-R_t(f_t)\bigr|,
    \qquad
    V_T=\frac1T\sum_t v_t,
\]
and $\Delta_T^{\mathrm{rep}}$ using held-out empirical loss and the one-step-ahead finite-population target.

For each fixed categorical channel $K$, rows are mapped to discrete states by binning selected coordinates, yielding histograms $\hat q_t^{(K)}$ on a finite simplex. The categorical Fisher--Rao step is
\[
    d_F(\hat q_{t+1}^{(K)},\hat q_t^{(K)})
    =
    2\arccos\!\left(
        \sum_j
        \sqrt{\hat q_{t+1,j}^{(K)}\hat q_{t,j}^{(K)}}
    \right),
\]
with a small pseudocount to avoid empty bins, and the observed rate is
\[
    \mathcal A_T^{(K)}/T
    =
    \frac1T\sum_t
    d_F(\hat q_{t+1}^{(K)},\hat q_t^{(K)}).
\]

The channels are: null blind, a fixed row-bucket control; coarse score, which bins only the deployed prediction score; task-aligned, which bins prediction score, order quantity, unit cost, and sales-channel subgroup; and three ablations omitting quantity, cost, or subgroup.

Run-level diagnostics fit
\[
    Y_T=\beta_0+\beta_1 \mathcal A_T^{(K)}/T+\varepsilon,
    \qquad
    Y_T\in\{V_T,\Delta_T^{\mathrm{rep}}\},
\]
separately for each channel. Transition-level diagnostics report Spearman association between
\[
    d_F(\hat q_{t+1}^{(K)},\hat q_t^{(K)})
    \quad\text{and}\quad
    v_t
\]
over positive-feedback transitions. Raw categorical rates are not calibrated estimates of $C_T/T$, $V_T$, or run-level degradation.

\section{Additional Experimental Results}
\label{app:exp-results}

\begin{figure}[H]
  \centering
  \includegraphics[width=0.6\linewidth]{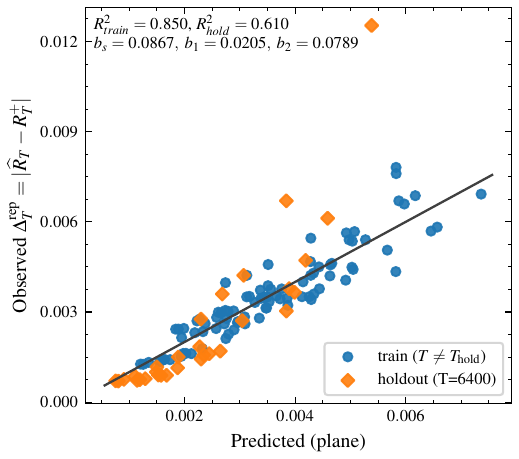}
  \caption{
  Neural-network plane-fit diagnostic. Observed $\Delta_T^{\mathrm{rep}}$ versus
  predictions from the additivity plane trained on $T\neq T_{\mathrm{hold}}$, with
  $T_{\mathrm{hold}}=6400$ highlighted.
  }
  \label{fig:nn_obs_pred}
\end{figure}

\bibliographystyle{plainnat} 
\bibliography{main}

\end{document}